\documentclass{article}

\usepackage[preprint]{neurips_2025}
\makeatletter
\renewcommand{\@noticestring}{Accepted at NeurIPS 2025}
\makeatother
\usepackage{preamble}
\usepackage[utf8]{inputenc} 
\usepackage[T1]{fontenc}    
\usepackage{amsthm}

\newcommand{\MLE}{\text{MLE}}

\newcommand{\F}{\mathcal{F}}
\newcommand{\A}{\mathcal{A}}

\newcommand{\real}{\mathbb{R}}

\newcommand{\DD}{\mathcal{D}}
\newcommand{\dd}{\mathrm{d}}
\newcommand{\OO}{\mathcal{O}}

\newcommand{\DKL}[2]{\DD_{\text{KL}}\left(#1\middle\|#2\right)}
\newcommand{\II}[1]{\mathbb{I}_{\left\{#1\right\}}}
\newcommand{\PP}[1]{\mathbb{P}\left[#1\right]}

\newcommand{\EE}[1]{\mathbb{E}\left[#1\right]}

\newcommand{\EEs}[2]{\mathbb{E}_{#2}\left[#1\right]}

\newcommand{\PPt}[1]{\mathbb{P}_t\left[#1\right]}

\newcommand{\EEc}[2]{\mathbb{E}\left[#1\middle|#2\right]}

\def\argmin{\mathop{\mbox{ arg\,min}}}
\def\argmax{\mathop{\mbox{ arg\,max}}}
\DeclareMathOperator\supp{supp}

\newcommand{\siprod}[2]{\langle#1,#2\rangle}

\newcommand{\biprod}[2]{\bigl\langle#1,#2\bigr\rangle}

\newcommand{\norm}[1]{\left\|#1\right\|}
\newcommand{\abs}[1]{\left|#1\right|}

\newcommand{\ev}[1]{\left\{#1\right\}}
\newcommand{\pa}[1]{\left(#1\right)}
\newcommand{\bpa}[1]{\bigl(#1\bigr)}
\newcommand{\Bpa}[1]{\Bigl(#1\Bigr)}

\newcommand{\wh}{\widehat}
\newcommand{\wt}{\widetilde}

\usepackage{todonotes}
\definecolor{PalePurp}{rgb}{0.66,0.57,0.66}

\definecolor{whatever}{rgb}{0.13,0.57,0.36}

\newcommand{\IR}{\text{IR}\xspace}
\newcommand{\IG}{\text{IG}\xspace}
\newcommand{\SIG}{\overline{\IG}\xspace}
\newcommand{\SIR}{\overline{\IR}\xspace}
\newcommand{\SIRs}{\overline{\IR}^{(2)}\xspace}
\newcommand{\SIRc}{\overline{\IR}^{(3)}\xspace}

\newcommand{\lambdas}{\lambda^{(2)}}
\newcommand{\lambdac}{\lambda^{(3)}}
\newcommand{\Rs}{R^{(2)}}
\newcommand{\Rc}{R^{(3)}}

\newlength{\leftstackrelawd}
\newlength{\leftstackrelbwd}
\def\leftstackrel#1#2{\settowidth{\leftstackrelawd}%
{${{}^{#1}}$}\settowidth{\leftstackrelbwd}{$#2$}%
\addtolength{\leftstackrelawd}{-\leftstackrelbwd}%
\leavevmode\ifthenelse{\lengthtest{\leftstackrelawd>0pt}}%
{\kern-.5\leftstackrelawd}{}\mathrel{\mathop{#2}\limits^{#1}}}

\newcommand{\Llik}{L^{(1)}}

\newcommand{\Lest}{L^{(2)}}
\newcommand{\Deltah}{\wh{\Delta}}

\newcommand{\pIDS}{\pi^{(\mathbf{SOIDS})}}
\newcommand{\pFGTS}{\pi^{(\mathbf{FGTS})}}
\newcommand{\pmix}{\pi^{(\mathbf{mix})}}

\makeatletter
\@ifundefined{thm}{\newtheorem{thm}{Theorem}}{}
\@ifundefined{lemma}{\newtheorem{lemma}{Lemma}}{}
\@ifundefined{cor}{}{}
\@ifundefined{proposition}{\newtheorem{proposition}{Proposition}}{}
\makeatother

\theoremstyle{definition}
\makeatletter
\@ifundefined{definition}{\newtheorem{definition}{Definition}}{}
\makeatother

\title{Sparse Optimistic Information Directed Sampling}
\author{Ludovic Schwartz ~~~Hamish Flynn ~~~Gergely Neu\\
Universitat Pompeu Fabra, Barcelona, Spain\\
\texttt{\{ludovic.schwartz, hamishedward.flynn, gergely.neu\}@upf.edu}
}
\date{}

\begin{document}

\maketitle
\begin{abstract}

Many high-dimensional online decision-making problems can be modeled as stochastic sparse linear bandits. Most existing 
algorithms are designed to achieve optimal worst-case regret in either the data-rich regime, where polynomial dependence 
on the ambient dimension is unavoidable, or the data-poor regime, where dimension-independence is possible at the cost 
of worse dependence on the number of rounds. In contrast, the sparse Information Directed Sampling (IDS) algorithm
satisfies a Bayesian regret bound that has the optimal rate in both regimes simultaneously. In 
this work, we explore the use of Sparse Optimistic Information Directed Sampling (SOIDS) to achieve the same adaptivity in 
the worst-case setting, without Bayesian assumptions. Through a novel analysis that enables the use of a time-dependent 
learning rate, we show that SOIDS can optimally balance information and regret. Our 
results extend the theoretical guarantees of IDS, providing the first algorithm that simultaneously achieves optimal 
worst-case regret in both the data-rich and data-poor regimes. We empirically demonstrate the good performance of SOIDS.

\end{abstract}

\begin{small}
\end{small}
\section{Introduction}

In stochastic linear bandits, one assumes that the mean reward associated with each action is linear in an unknown 
$d$-dimensional parameter vector \citep{abe1999associative, auer2002using, dani2008stochastic, abbasi2011improved}.
Under standard conditions, it is known that the minimax regret in this setting is of the order $ \OO(d \sqrt{T}) $ \citep{dani2008stochastic, rusmevichientong2010linearly}. Numerous follow-up works have investigated the possibility of reduced regret under various structural assumptions on the unknown parameter vector, the noise, or the shape of the decision 
set \citep{Valko_M_K_K14, Chu_L_R_S11, Kirsc_K18}, \citep[Chapter 22]{lattimore2020bandit}. One such assumption is that the unknown parameter vector is \emph{sparse}, which means that it has only $ s \ll d $ non-zero components. This setting is called \emph{sparse linear bandits} and $s$ is referred to as 
the \emph{sparsity level}. In this setting, previous work has established the existence of algorithms with regret 
scaling as $ \OO(\sqrt{sdT}) $ \citep{Abbas_P_S12}. This result is complemented by a lower bound, which says that this rate cannot be improved as long as $ T \geq d^{\alpha} $ for some $ \alpha>0$ \citep{lattimore2020bandit}. We refer to this scenario as the
\textit{data-rich regime}. Since this bound scales polynomially with the dimension $d$, many researchers have 
considered this to be a negative result, interpreting it as a sign that sparsity cannot be effectively exploited 
in linear bandit problems. This interpretation has been challenged by a more recent observation that, when the action
set admits an \textit{exploratory distribution}, simple ``explore-then-commit'' algorithms enjoy regret bounds of order
$ \OO((sT)^\frac{2}{3})$ \citep{Hao_L_W20, jang2022popart}. These bounds scale only logarithmically with 
the dimension, and constitute a major improvement over the previously mentioned rate in the \textit{data-poor
regime}, where $ T \ll \left( \frac{d}{s} \right)^3 $. 
Most known algorithms are specialized to either the data-poor or data-rich regime, and perform poorly in the other one. A notable exception is the \emph{sparse Information Directed Sampling} algorithm introduced in \citet{Hao_L_D21a}, which 
performs almost optimally in both regimes. However, \citet{Hao_L_D21a} only provide \emph{Bayesian} performance guarantees for sparse IDS. These results hold on average, assuming that the problem instance is drawn at random from a known prior distribution.


In this work, 
we lift this assumption and develop an algorithm that can adapt to both regimes in a ``frequentist'' sense: we 
assume that the true parameter is fixed and unknown to the learner, and provide guarantees that hold for any given instance.
The algorithm is an adaptation of the recently proposed Optimistic Information Directed
Sampling (OIDS) algorithm of \citet*{Neu_P_S24}, which itself is an adaptation of the classic Bayesian IDS algorithm originally proposed by \citet{Russo_V17}. Within the Bayesian setting, it has been shown that IDS can exploit various 
types of problem structure, and adapt to the hardness of the given instance \citep{Hao_L22a, Hao_L_Q22}. These results have been
complemented by the recent work of \citet*{Neu_P_S24}, which showed that similar improvements can be 
achieved without Bayesian assumptions, via a simple adjustment of the standard IDS method. In this paper, we continue 
this line of work. We show that a sparse version of OIDS enjoys a worst-case regret bound that matches the optimal rate in both regimes simultaneously.


Our contribution is as follows:
\begin{itemize}

	\item We extend the analysis of the optimistic posterior to allow the use of 
	time-dependent learning rates and history-dependent learning rates. Time-dependent learning rates allow us to drop the requirement that the horizon is known in advance, and are essential for worst-case regret bounds that can adapt to both regimes. History-dependent learning rates allow us to update the 
	learning rate based on data observed by the agent instead 
of some loose
		theoretical constant, a necessity for efficient algorithms.
	\item We demonstrate that the Sparse Optimistic Information Directed Sampling (SOIDS) algorithm
		recovers almost optimal rates in
		both the data-poor and data-rich regimes. This is the first algorithm to do so in a frequentist setting.
\end{itemize}

\section{Preliminaries}


\paragraph{Sparse linear bandits.} We consider the following decision-making game, in which a learning agent interacts with an environment over a sequence of $T$ rounds. At the start of each round $t$, the learner selects an action $A_t \in \mathcal{A} \subset \mathbb{R}^d$ according to a randomized policy $\pi_t \in \Delta(\A)$. In response, the environment generates a stochastic reward $Y_t = r(A_t) + \epsilon_t$, where $r: \A \to \real$ is a fixed reward function and $\epsilon_t$ is zero-mean, conditionally 1-sub-Gaussian noise. We assume that the action set $\mathcal{A}$ is finite, and that the reward function can be written as
\begin{equation*}
r(a) = \siprod{\theta_0}{a}\,,
\end{equation*}
where $\theta_0 \in \mathbb{R}^d$ is an unknown parameter vector. We make the mild boundedness assumptions that $\max_{a \in \A}\|a\|_{\infty} \leq 1$ and $\|\theta_0\|_1 \leq 1$. We study the special case of this problem in which the parameter vector $\theta_0$ is 
$s$-sparse in the sense that at most $s \ll d$ of its components are non-zero. In other words, we assume that $\theta_0$ belongs to the following \emph{sparse 
parameter space}:
\begin{equation*}
\Theta = \left\{ \theta \in \mathbb{R}^d : {\textstyle\sum_{j=1}^{d}} \II{\theta_j \neq 0} \leq s,\ \|\theta\|_1 
\leq 1 \right\}\,.
\end{equation*}
We assume that the sparsity level $s$ is known to the agent. The performance of the agent is evaluated in terms of the \emph{regret}, which is defined as
\begin{equation}
\label{eq:regret}
R_T = T \max_{a \in \mathcal{A}} \siprod{\theta_0}{a} - \mathbb{E} \left[ \sum_{t=1}^{T} r(A_t, \theta_0) \right]\,,
\end{equation}
where the expectation is taken with respect to both the random choices of the agent and the random noise in the 
observed rewards. We note that the regret is implicitly a function of the true parameter $ \theta_0$. Our focus is on proving regret bounds that hold for arbitrary choices of $\theta_0 \in \Theta $.

\paragraph{The data-rich and data-poor regimes.} As mentioned in the introduction, it is known there exist algorithms for sparse linear bandits with worst-case regret of the order $\OO(\sqrt{sdT})$ \citep{Abbas_P_S12}. This regret bound is only meaningful when the dimension $d$ is smaller than the number of rounds $T$, a situation referred to as the data-rich regime. Under the assumption that there exists an exploratory policy, \citet{Hao_L_W20} showed that there is a simple algorithm that satisfies a problem-dependent regret bound, which can be meaningful in the so-called data-poor regime, where $d$ is much larger than $T$. Formally, we say that there exists an exploratory policy if the action set $\A$ is such that
\begin{equation*}
C_{\min} := \max_{\mu \in \Delta(\A)} \sigma_{\min}\left(\int_{\A} a a^{\top} \, d \mu(a)\right) > 0\,,
\end{equation*}

which is equivalent to the condition that $\A$ spans $\mathbb{R}^d$. The exploratory policy is the distribution on $\A$ that achieves the maximum (which is guaranteed to exist when $\A$ is finite). The Explore the Sparsity Then Commit (ESTC) algorithm was shown to satisfy a regret bound of the order $\OO(s^{2/3}T^{2/3}C_{\min}^{-2/3})$ \citep{Hao_L_W20}. The transition between the $T^{2/3}$ rate in the data-poor regime and the $\sqrt{T}$ rate in the data-rich regime also appears in an existing lower bound of the order $\Omega(\min(s^{1/3}T^{2/3}C_{\min}^{-1/3}), \sqrt{dT})$ \citep{Hao_L_W20}.

\paragraph{Adapting to both regimes.} Recently, \citet{Hao_L_D21a} showed that the sparse Information Directed Sampling (IDS) algorithm performs well in both regimes. Under the sparse optimal action condition (Definition~\ref{def:sparse_optimal_action}), IDS satisfies a regret bound of the order $\OO(\min(\sqrt{dT\Delta}, (sT)^{2/3}\Delta^{1/3}C_{\min}^{-1/3}))$, where $\Delta \propto \min(\log(|\A|), s\log(dT/s))$. This is simultaneously optimal in both the data-rich and data-poor regimes. However, this result is limited to the 
Bayesian setting. This is because IDS uses the Bayesian posterior to quantify uncertainty, which is only meaningful if $\theta_0$ really is a random draw from the prior.

\paragraph{The sparse optimal action condition.}
Part of our analysis requires that a certain technical condition is satisfied. This condition comes from prior work \citep{Hao_L_D21a}, and is used to bound the
regret in the data-poor regime (cf.~Lemma \ref{lemma:IR_bound}).
\begin{definition}
\label{def:sparse_optimal_action}
For a given prior $Q_1^+$, an action set $\A$ has sparse optimal actions if with probability 1 over the random draw of $\theta$ from $Q_1^+$, there exists $a^{\prime} \in \argmax_{a \in \A}r(a, \theta)$ such that $\|a^{\prime}\|_0 \leq s$.
\end{definition}

We use a prior that only assigns positive probability to $s$-sparse vectors, which means the sparse
optimal action property
is satisfied whenever the action set is an $\ell_p$-ball.
Note that the hard instances in both the $\sqrt{sdT}$ lower bound in Theorem 24.3 of \cite{lattimore2020bandit} and
the $s^{2/3}T^{2/3}$ lower bound in Theorem 5 of \citet{jang2022popart} satisfy the sparse optimal action
property\footnote{The optimal actions in the hard instance used to prove Theorem 5 in \citet{jang2022popart} are
$2s$-sparse, which still allows us to prove the same bound on the surrogate 3-information ratio, up to constant
factors.}. Therefore, imposing this additional condition does not trivialize the problem.

\paragraph{Notation.} We conclude this section by introducing some additional notation that will be used in the subsequent sections. For any candidate parameter vector (or model) $\theta \in \mathbb{R}^d$, we let $r(a, \theta) = \siprod{\theta}{a}$ denote the corresponding linear reward function. In addition, we define $a^*(\theta) = \argmax_{a \in \A} r(a, \theta)$ (with ties broken arbitrarily) and $r^*(\theta) = r(a^*(\theta), \theta)$ to be the optimal action and maximum reward for the model $\theta$. The gap of an action $a$ for a model $\theta$ is $ \Delta (a, \theta) = r^{*}(\theta) - r(a, \theta) $.
Similarly, the gap for a policy $\pi \in \Delta(\A)$ and a model distribution $ Q\in \Delta(\Theta) $ is $\Delta(\pi,Q)
= \int_{\A\times \Theta}\Delta (a, \theta)\,\mathrm{d}\pi\otimes Q(a,\theta)$, and we let $\Delta_t = \Delta(\pi_t, \theta_0)$ denote the gap of the policy played by the agent in round $t$ under the true model $\theta_0$. Using this notation, the regret can be written as $R_T = \mathbb{E}[\sum_{t=1}^{T}\Delta_t]$. We define the unnormalized Gaussian likelihood function $p(y|\theta, a) = \exp(-\frac{(y - \siprod{\theta}{a})^2}{2})$. Finally, we let $\mathcal{F}_t = \sigma(A_1, Y_1, \ldots, A_t, Y_t)$ denote the $\sigma$-algebra generated by the interaction between the agent and the environment up to the end of round $t$. 

\section{Sparse Optimistic Information Directed Sampling}

We develop an extension of the Optimistic Information Directed Sampling (OIDS) algorithm proposed by 
\citet*{Neu_P_S24}. The main difference between OIDS and IDS is that the
Bayesian posterior is replaced by an appropriately adjusted \emph{optimistic posterior}. For an arbitrary prior 
$ Q_1^{+} \in
\Delta(\Theta) $, the optimistic posterior is defined by the following update rule:
\begin{equation}
\label{eq:update_Optimistic_posterior}
	\frac{dQ_{t+1}^{+}}{dQ_1^{+}}(\theta) \propto \prod_{s=1}^{t}(p(Y_s\mid \theta, A_s))^ \eta \cdot \exp\Bpa{\lambda_t
	\sum_{s=1}^{t} \Delta(A_s, \theta)}. 
\end{equation}
Here, $ \eta $ is a positive constant that should be thought of as ``large'', and $ (\lambda_t)_t $ is a decreasing sequence of
positive real numbers that decays to $ 0 $, and should be thought of as ``small''. 
We allow $ \lambda_t $ to be computed by the 
algorithm at the end of the round $t$. In other words, any $ \F_t $-measurable $ \lambda_t $
is admissible. Note that when $ \eta=1 $ and $
\lambda_t = 0 $, the optimistic posterior coincides with the Bayesian posterior.
While this construction is closely related to the optimistic posterior update described in \citet{Zhang_22} and 
\citet*{Neu_P_S24}, there are a few important differences. First, the $\Delta(A_s,\theta)$ term appearing in the 
adjustment serves as an alternative to their proposal of using $r^*(\theta)$ for the same purpose. 
Intuitively this serves to ``overestimate'' the true gaps with the optimistic posterior, driving exploration towards 
parameters that promise rewards much higher than whatever would have been accrued by the agent. In contrast, the 
adjustment of \citet{Zhang_22} drives exploration towards parameters $\theta$ with high optimal reward regardless of 
how well the agent would have performed under the same $\theta$---meaning that it unduly assigns mass to 
uninteresting parameter choices, where any policy is guaranteed to work well anyway. 
Intuition aside, this adjustment greatly simplifies our analysis of the optimistic posterior as compared to the analysis of \citet{Zhang_22} and 
\citet*{Neu_P_S24}. An important additional novelty is that our update features a time-dependent exploration parameter 
$\lambda_t$, which is crucial for the adaptive regret bounds that we seek in this work.
%
To describe the OIDS algorithm, we must first define the \textit{surrogate
information gain} and the \textit{surrogate regret}. For any round $t$ and any policy $\pi \in \Delta(\A)$, the surrogate information gain is defined as
\begin{equation*}
\SIG_t(\pi) = \frac{1}{2}\sum_{a \in \A} \pi(a) \int_{\Theta}\pa{\siprod{\theta- \bar{\theta}(Q_t^+)}{a} }^2 \, 
\dd Q_t^+(\theta)\,,\label{eq:SIG}
\end{equation*}
where for any $Q \in \Delta(\Theta)$, $ \bar{\theta}(Q) = \EEs{\theta}{\theta \sim Q}$ is the mean parameter under distribution $Q$. The surrogate regret is defined as
\begin{equation*}
\label{eq:SR}
\Deltah_t(\pi) = \sum_{a \in \A} \pi(a) \int_{\Theta} \Delta(a, \theta) \, \dd Q_t^+(\theta)\,.
\end{equation*}

For
any policy $ \pi $ and any $\gamma \geq 2 $, 
we define the \textit{surrogate generalized information ratio} as
\begin{equation}
\label{eq:SIR}
	\SIR_t^{(\gamma)}(\pi) = \frac{(\Deltah_t(\pi))^{\gamma}}{\SIG_t(\pi)} = 2\cdot  \frac{\left( \sum_{a \in \A} 
\pi(a)
	\int_{\Theta}\siprod{\theta}{a^{*}(\theta) - a} \, dQ_t^{+}(\theta)\right)^{\gamma}}{\sum_{a\in \A} \pi(a)
\int_{\Theta} (\siprod{\theta- \bar{\theta}(Q_t^{+})}{a})^2 \, dQ_t^{+}(\theta)}.
\end{equation}
We can at last define the Sparse Optimistic Information Directed Sampling (SOIDS) algorithm. In each 
round $t$, the policy played by SOIDS is defined to be the distribution on $\A$ that minimizes the $2$-information ratio:
\begin{equation}
\pIDS_t = \argmin_{\pi \in \Delta(\A)} \SIRs_t(\pi)\,.\label{eq:soids_policy}
\end{equation}
The choice of $ \gamma=2 $ is motivated by the fact that the minimizer of the $ 2
$-information ratio is an approximate minimizer of surrogate generalized information ratio for all $ \gamma \geq 2 $.
\begin{lemma}
\label{lemma:GIR_minimizer}
For all $ \gamma \geq 2 $,
\begin{equation*}
	\SIR_t^{(\gamma)}(\pIDS_t)\leq 2^ {\gamma-2}\min_{\pi\in\Delta(\A)}\SIR_t^{(\gamma)}(\pi)\,.
\end{equation*}
\end{lemma}
This fact was discovered for the Bayesian IDS policy by \citet{Latti_G21} and remains true for the SOIDS policy. We provide a proof in Appendix~\ref{app:GIR_minimizer} for completeness.
Finally, we remark that the "sparse" part of the name SOIDS refers to the choice of the prior $Q_1^+$. We use the subset selection prior from Section 3 of \citet{alquier2011pac}, which is described in Appendix~\ref{app:prior_comparator_bound}.

\section{Main results}\label{sec:main}
In this section, we state our main results. First, we relate the true regret of any policy
sequence to the surrogate regret of the same policy sequence. We then use the fact that the surrogate regret is controlled by both the $ 2 $- and $ 3 $-information ratios. This, combined with Lemma \ref{lemma:GIR_minimizer}, allows us to show that with properly tuned parameters, SOIDS has optimal worst-case regret in both the data-poor and data-rich regimes.
Finally, we show that SOIDS can be tuned in a data-dependent manner, such that its regret bound scales with the cumulative observed
information ratio instead of the time horizon.

\subsection{General bound for the optimistic posterior}
We start with a generic worst-case regret bound relating the true regret of any algorithm to its surrogate regret. Since the surrogate regret is defined with respect to the optimistic posterior, which is
known to the learner, it can be controlled with standard Bayesian techniques. This result is an extension of the bounds stated in
\citet{Neu_P_S24}, \citet{Zhang_22}. To our knowledge, it is the first result of
its kind which is compatible with time-dependent or data-dependent learning rates. The
stated result is specialized to the setting of sparse linear bandits, but the techniques used to deal with
time-dependent and data-dependent learning rates are applicable beyond this setting.
\begin{thm}
\label{thm:Optimistic_posterior_RB}
Assume that the optimistic posterior is computed with $ \eta = \frac{1}{4} $ and a sequence of decreasing learning rates
$ \lambda_t $ satisfying $ \forall t\geq 1, \lambda_t \leq \frac{1}{2} $. Set $ \lambda_0 = \frac{1}{2} $.
If the learning rates do not depend on the history, then the regret of any sequence of policies $ \pi_t $ satisfies
\begin{equation}
\label{eq:Optimistic_posterior_RB_indep}
R_T \leq \EE{\frac{5 + 2s \log \frac{edT}{s}}{\lambda_{T-1}} -\sum_{t=1}^{T}
\frac{3}{32}\cdot \frac{\SIG_t(\pi_t)}{\lambda_{t-1}} + 2\sum_{t=1}^{T} \Deltah_t(\pi_t)}.  
\end{equation}

Otherwise, if the learning rates depend on the history, let $ C_{1,T} $
be a deterministic upper bound on $ \frac{1}{\lambda_t} - \frac{1}{\lambda_{t-1}} $ valid for all $ t \leq T $, and $
C_{2,T} $ be a deterministic upper bound on $ \frac{1}{\lambda_{T-1}} $. The regret of any sequence
of policies $ \pi_t $ satisfies
\begin{equation}
\label{eq:Optimistic_posterior_RB_dep}
R_T \leq \EE{\frac{2 + s \log \frac{4e^3d^2 T^3 C_{1,T}^2 C_{2,T} }{s^2}}{\lambda_{T-1}} -\sum_{t=1}^{T}
\frac{3}{32}\cdot \frac{\SIG_t(\pi_t)}{ \lambda_{t-1}} + 2\sum_{t=1}^{T} \Deltah_t(\pi_t)} + 2.  
\end{equation}
\end{thm}

\subsection{Adapting to both regimes}
We show that the SOIDS algorithm with properly tuned parameters achieves the optimal regret rate in both the data-rich and data-poor
regimes.

\begin{thm}
\label{thm:regret_main}
Assume that the sparse optimal action condition in
Definition~\ref{def:sparse_optimal_action} is satisfied.
Let $ \lambdas_t =  \sqrt{\frac{3C_{t+1}}{128d(t+1)}} $ and $ \lambdac_t = \frac{1}{4\cdot  6 ^\frac{1}{3}} \left(
	\frac{C_{t+1}\sqrt{C_{\min}}}{(t+1) \sqrt{s}}
\right)^{\frac{2}{3}} $, with $ C_t = 5 + 2s \log \frac{ ed t}{s} $.
If $ \lambda_t = \min(\frac{1}{2},
\max(\lambdas_t, \lambdac_t)) $, then the regret of SOIDS satisfies
\begin{align}
\label{eq:regret_main}
R_T &\leq \min \left( 27 \sqrt{\left( 5 + 2s \log \frac{edT}{s} \right)d T}, 30 \left( 5 + 2s \log\frac{edT}{s}
\right)^{\frac{1}{3}} \left( \frac{T \sqrt{s}}{\sqrt{C_{\min}}} \right)^{\frac{2}{3}} \right) + \OO(\sqrt{s} \log
\frac{d}{\sqrt{s}})\\
\nonumber	 &= \min\left(\OO\left(\sqrt{sdT \log \frac{edT}{s}}\right), \OO\left((s T)^{\frac{2}{3}}\left( \log \frac{edT}{s} \right)^\frac{1}{3}\right)\right),
\end{align}
where $ \OO(\sqrt{s} \log \frac{d}{\sqrt{s}}) $ represents an absolute constant independent of T.
\end{thm}
We observe that our algorithm enjoys both the $ \widetilde{\OO}(\sqrt{sdT}) $ and the $ \widetilde{\OO}((s T)^\frac{2}{3})
$ regret rates.
Unlike the Bayesian regret bound for the sparse IDS algorithm of \citet{Hao_L_D21a}, our regret bound holds in a ``worst-case'' sense for any value of $\theta_0 \in \Theta$. To our knowledge, this makes our method the first algorithm to achieve optimal worst-case regret in both 
the data-poor and data-rich regimes

\subsection{Instance-dependent guarantees}
The bounds presented in the previous sections are minimax in nature, meaning they hold uniformly over all problem
instances. We present a bound in which
the scaling with respect to the horizon $T$ is replaced with the cumulative surrogate-information ratio, which could be much smaller than $T$ in ``easier'' instances, leading to better guarantees.
\begin{thm}
\label{thm:instance_dependent_regret}
Assume that the sparse optimal action condition in
Definition~\ref{def:sparse_optimal_action} is satisfied, and that $ s \leq \frac{d}{2} $. Let $ \lambdas_t = \sqrt{\frac{s}{2d +
	\sum_{s=1}^{t} \SIRs_s(\pi_s)}} $ and $ \lambdac_t = \bigg(
\frac{s}{ \frac{3 \sqrt{6} s}{ \sqrt{C_{\min}}} + \sum_{s=1}^{t} \sqrt{\SIRc_s(\pi_s)}} \bigg)^\frac{2}{3} $. If $ \lambda_t = \max (\lambdac_t, \lambdas_t)$, then the regret of SOIDS satisfies
\begin{align}
\label{eq:instance_dependent_regret}
R_T &\leq \left(\frac{2}{s} + \frac{80}{3} + 5\log \frac{edT}{s}\right) \min \left( \sqrt{s \left( 2d + \sum_{t=1}^{T-1}
\SIRs_t(\pi_t) \right)}, s^\frac{1}{3}  \left( \frac{3 \sqrt{6} s}{ \sqrt{C_{\min}}} + \sum_{t=1}^{T}
\sqrt{\SIRc_t(\pi_t)} \right)^\frac{2}{3}  \right)\\
    \nonumber&= \OO\left( \log \frac{edT}{s} \min \left( \sqrt{s \left( 2d + \sum_{t=1}^{T-1}
\SIRs_t(\pi_t) \right)}, s^\frac{1}{3}  \left( \frac{3 \sqrt{6} s}{ \sqrt{C_{\min}}} + \sum_{t=1}^{T}
\sqrt{\SIRc_t(\pi_t)} \right)^\frac{2}{3}  \right)\right).
\end{align}
\end{thm}

This type of result is only possible because our novel analysis of the optimistic posterior (cf. Theorem \ref{thm:Optimistic_posterior_RB}) can handle history-dependent learning rates. A full proof is provided in Appendix~\ref{app:instance_dependent_regret}. This result shows that (with appropritate choices of the learning rates) SOIDS is
fully adaptive to which of the two regimes is best. Because our analysis requires decreasing learning rates, we are
forced to leave the $ \log(T) $ terms out of the learning rates, and our logarithmic term has a worse power than in the bound
of Theorem~\ref{thm:regret_main}. An
interesting open question is whether it is possible to improve the dependency on this logarithmic term while still using
data-dependent learning rates.

\section{Analysis}
We now provide an outline of the proofs of the main results.

\subsection{Proof of Theorem~\ref{thm:Optimistic_posterior_RB}}

A key observation is that the optimistic posterior can be interpreted as a learner playing an auxiliary online
learning game over distributions $ \Delta(\Theta) $. The loss of that game is a weighted sum
of negative log-likelihood and estimation error losses. We define
\begin{equation*}
\Llik_t(\theta) = \sum_{s=1}^{t} \log \pa{ \frac{1}{p(Y_s|\theta, A_s)}}= \sum_{s=1}^{t}\frac{1}{2}
		\bpa{\siprod{\theta}{A_s} - Y_s}^2
\end{equation*}
to be the \emph{cumulative negative log-likelihood loss} of $ \theta $ and
\begin{equation*}
\Lest_t(\theta) = \sum_{s=1}^{t} - \Delta(A_s, \theta)
\end{equation*}
to be the \emph{cumulative estimation error loss} of $ \theta $. In addition, we define the regularizer $\Phi :\Delta(\Theta) \to \mathbb{R}$ by the mapping $P \mapsto \DKL{P}{Q_1^{+}}$, 
which is the KL-divergence with respect to the prior $ Q_1^{+} $. With those notations, the optimistic posterior corresponds to an instance of the Follow the
Regularized Leader (FTRL) algorithm introduced by \citet{Hazan_K10} and \citet{Abern_H_R08}. FTRL is a standard method
in online convex optimization that balances cumulative loss minimization with a regularization term to enforce stability
and guarantee controlled regret. The update can be
reframed as
\begin{equation*}
	Q_{t+1}^{+} = \argmin_{P\in \Delta(\Theta)} \siprod{P}{\eta \Llik_t + \lambda_t \Lest_t} + \Phi(P).
\end{equation*}

This formulation enables the application of tools from convex analysis and online learning, such as Fenchel duality, to derive
regret bounds for this auxiliary online learning game and to understand the interplay between the two losses under the learning rates $ \eta $ and $ \lambda_t
$. We now focus on the case in which the learning rates $ \lambda_t $ don't depend on the history and relegate the
analysis of history-dependent learning rates to
Appendix~\ref{app:history_dep_Optimistic_posterior}. The following lemma provides a bound on the average regret when the model $\theta_0$ is drawn from an arbitrary comparator distribution $ P $.

\begin{lemma}
\label{lemma:FTRL_decomposition}
Let $ P\in \Delta(\Theta) $ be any comparator, then the following bound holds	
\begin{equation}
\label{eq:FTRL_decomposition}
	 \sum_{t=1}^{T} \Delta(A_t, P) \leq \frac{\Phi(P) + \Phi^{*}(\eta(\Llik_T(\theta_T) - 
\Llik_T(\cdot )) -
	\lambda_T \Lest_T(\cdot ))}{\lambda_T} + \frac{\eta}{\lambda_T}(\siprod{P}{\Llik_T} - \Llik_T(\theta_T))\,.
\end{equation}
\end{lemma}
Here $ \theta_t = \argmin_{\theta \in \Theta} \Llik_t(\theta) $ denotes the maximum likelihood estimator at time t, and 
$ \Phi^{*}(L) = \log
\int_{\Theta}\exp(L(\theta))dQ_1^{+}(\theta) $ is the Fenchel dual of the regularizer $ \Phi $. A complete proof of this result is provided in appendix~\ref{app:FTRL_decomposition}. We aim to chose a comparator $ P$ and the prior $ Q_1^{+} $ such that $ P $ is concentrated around $
\theta_0 $ and the KL divergence $ \DKL{P}{Q_1^{+}} $
is controlled. If the parameter space were finite, the natural choice would be to take $ P $ as a
Dirac on $ \theta_0 $ and $ Q_1^{+} $ as a uniform distribution on the whole parameter space; more care is necessary here. 
Choosing $ Q_1^{+} $ as a subset-selection prior and $P$ as a uniform distribution on a sparse 
neighborhood of $ \theta_0 $ satisfies both requirements.

\begin{lemma}
\label{lemma:prior_comparator_bound}
The subset-selection prior $ Q_1^{+} \in \Delta(\Theta)$ verifies that for any $ \epsilon>0 $  and $ \theta \in \Theta$, there
is a comparator $ P_{\theta}
\in \Delta(\Theta)$ satisfying both 
\begin{equation*}
	\forall \theta' \in \supp(P_{\theta}), ~\norm{\theta - \theta'}_1 \leq \epsilon \quad \text{and} \quad
\DKL{P_{\theta}}{Q_1^+} \leq s \log \frac{2ed}{\epsilon s} .
\end{equation*}
\end{lemma}

The proof of this lemma, as well as the exact choice of the prior $ Q_1^+ $ and the comparator $ P(\theta_0) $, are provided
in Appendix~\ref{app:prior_comparator_bound}. In Appendix~\ref{app:technical_results} (cf.
Lemma~\ref{lemma:lipschitzness}), we establish that both $ \Lest_T(\cdot ) $ and $
\EE{\Llik_T(\cdot)} $ are $ 2T $-Lipschitz with respect to the $ \ell_1
$-norm.
Hence, 
\begin{equation*}
	\EE{\frac{|P\cdot \Llik_T - \Llik_T(\theta_0)|}{\lambda_{T}}} \leq  \frac{2T \epsilon}{\lambda_T} , \quad \text{and} \quad \sum_{t=1}^T|\Delta(\theta_0,
	A_t) - \Delta(P, A_t) | \leq 2T \epsilon.
\end{equation*}

Combining these with Lemma~\ref{lemma:FTRL_decomposition}, we obtain the following bound on the
cumulative regret:
\begin{align*}
	R_T \leq&  \EE{\frac{s\log \frac{2 e d}{ \epsilon s} + 2T(\lambda_T + \eta) \epsilon }{\lambda_T} + 
\frac{\Phi^{*}(- \eta(\Llik_T(\cdot ) - \Llik_T(\theta_T)) -
	\lambda_T \Lest_T(\cdot ))}{\lambda_T}}\\
	&+ \EE{\frac{\eta}{\lambda_T}(\Llik_T(\theta_0) - \Llik_T(\theta_T))}\,.
\end{align*}

The first term balances model complexity and approximation via $ \epsilon $. In the usual FTRL analysis, $ \lambda \rightarrow \frac{\phi^{*}(\lambda L)}{\lambda} $ is non
decreasing for any $ L \in \real^{\Theta} $, and the term involving $ \Phi^{*} $ can be telescoped. Things are more
complex here because only some part of the loss is weighted by the time varying learning rate $ \lambda_T $. Through a careful
analysis involving the maximum likelihood estimator, we can decompose the $ \Phi^{*} $ term into a telescoping sum and a
remainder term.

\begin{lemma}
\label{lemma:pre_telescoping_phi_s}
\begin{align}
\label{eq:pre_telescoping_phi_s}
\nonumber&\phantom{=}\EE{\frac{\Phi^{*}(\eta (\Llik_T(\theta_T) - \Llik_T(\cdot )) - \lambda_T \Lest_T(\cdot ) 
)}{\lambda_T}}\\
	&\leq \EE{\sum_{t=1}^{T}\frac{\Phi^{*}(\eta (\Llik_t(\theta_{0}) \Llik_t(\cdot )) - \lambda_{t-1} \Lest_t(\cdot )
	)}{\lambda_{t-1}}
	-\frac{\Phi^{*}(\eta (\Llik_{t-1}(\theta_0)-\Llik_{t-1}(\cdot )) - \lambda_{t-1} \Lest_{t-1}(\cdot )
	)}{\lambda_{t-1}}}\\
	&\qquad+ \frac{\eta (6 + s \log \frac{edT}{s})}{\lambda_T}.
\end{align}
\end{lemma}
A detailed proof of this result is provided in Appendix~\ref{app:pre_telescoping_phi_s}. Finally, the remaining sum can be 
handled by looking at the explicit formula for $ \Phi^{*} $. The terms related to the likelihood and the gap estimates 
can be separated using
Hölder's inequality, as is done in \citet{Zhang_22} and \citet*{Neu_P_S24}. More explicitly, by now choosing $ \eta
= \frac{1}{4}$, we obtain the following lemma.

\begin{lemma}
\label{lemma:telescoping_phi_s}

\begin{align}
\label{eq:telescoping_phi_s}
\nonumber&\phantom{=} \EE{\sum_{t=1}^{T}\frac{\Phi^{*}(\eta (\Llik_t(\theta_{0}) - \Llik_t(\cdot )) - \lambda_{t-1} 
\Lest_t(\cdot )
	)}{\lambda_{t-1}}
	-\frac{\Phi^{*}(\eta (\Llik_{t-1}(\theta_0)- \Llik_{t-1}(\cdot )) - \lambda_{t-1} \Lest_{t-1}(\cdot )
	)}{\lambda_{t-1}}}\\
	 &\qquad \leq \EE{-  \sum_{t=1}^{T}\frac{3\SIG_t(\pi_t)}{32\lambda_{t-1}} + 2\sum_{t=1}^{T} \Deltah(\pi_t)}.
\end{align}
\end{lemma}
A full proof of this result is provided in Appendix~\ref{app:telescoping_phi_s}.
Combining Lemmas~\ref{lemma:FTRL_decomposition}, \ref{lemma:prior_comparator_bound}, \ref{lemma:pre_telescoping_phi_s},
\ref{lemma:telescoping_phi_s} and setting $ \epsilon = \frac{2}{T} $, we obtain the desired regret bound stated in
Theorem~\ref{thm:Optimistic_posterior_RB}. 

\subsection{Proof of Theorem~\ref{thm:regret_main}}
We show how Theorem \ref{thm:Optimistic_posterior_RB} can be combined with bounds on the surrogate regret to control the
true regret.
The first important fact is that the surrogate regret of any policy can always be controlled in terms of the $ 2 $ or
the $ 3
$-surrogate information
ratio of that policy.

\begin{lemma}
\label{lemma:regret_GIR_bound}
Let $ \lambda>0 $, then we have that for any policy $ \pi \in \Delta(\A) $
\begin{equation*}
	\Deltah_t(\pi) \leq \frac{\SIG_t(\pi)}{\lambda} + \min \left(\frac{1}{4} \lambda \SIRs_{t}(\pi), c_3^{*}\sqrt{\lambda
	\SIRc_t(\pi)}\right),
\end{equation*}
where $ c_3^{*} < 2 $ is an absolute constant defined in Lemma~\ref{lemma:Gen_AM_GM}.
\end{lemma}
This is a consequence of a simple generalization of the AM-GM inequality and is proved in
Appendix~\ref{app:regret_GIR_bound}. Combining the previous lemma with $ \lambda = \frac{64}{3} \lambda_{t-1} $ and Theorem~\ref{thm:Optimistic_posterior_RB}, we can further upper bound the regret
of a sequence of policies $ (\pi_t)_t $ as
\begin{align}
	\nonumber R_T &\leq \EE{\frac{5 + 2s \log \frac{edT}{s}}{\lambda_{T-1}} -\sum_{t=1}^{T}
\frac{3\SIG_t(\pi_t)}{32 \lambda_{t-1}} + 2\sum_{t=1}^{T} \Deltah_t(\pi_t)} \\
					  \label{eq:RB_max}&\leq \EE{\frac{C_T}{\lambda_{T-1}} + \sum_{t=1}^{T}
						  \min\left(\frac{32}{3} \lambda_{t-1}\SIRs_t(\pi_t), \frac{16}{3}
					  c_3^{*} \sqrt{3\lambda_{t-1}\SIRc_t(\pi_t)}\right)},
\end{align}
where $ C_T = 5 + 2s \log \frac{edT}{s} $.
Usually, bounds on the $ 2 $-information ratio can be converted to $ \OO(\sqrt{T}) $ bounds and bounds on the $ 3
$-information ratio can be converted to $ \OO(T^{\frac{2}{3}}) $ bounds. Hence we will use the $ 2 $-information ratio
to control the regret in the data-rich regime and the $ 3 $-information ratio to control the regret in
the data-poor regime.
Due to Lemma \ref{lemma:GIR_minimizer}, the SOIDS policy minimizes the $ 2 $-information ratio and approximately minimizes the $ 3 $-information ratio. As a
result, if there exists a "forerunner" algorithm with bounded $ 2 $-information ratio or $
3 $-information ratio, SOIDS inherits these bounds automatically. In particular, we can use a different forerunner for each regime and SOIDS will
match the regret guarantees of the best forerunner in each regime.

This forerunner-based technique is widely used to analyze IDS based
algorithms and has been applied to a variety of Bayesian settings \citep{Russo_V17, Hao_L_D21a, Hao_L22a} and some frequentist settings \citep{Kirsc_K18, Kirsc_L_K20, Kirsc_L_V_S21}.
An advantage of the OIDS framework is that since the surrogate quantities are defined with respect to the
optimistic posterior, the analysis of the surrogate information ratio is virtually identical to the corresponding
analysis of the information ratio in the Bayesian setting.

The forerunner we consider for the $ 2 $-information ratio is the \textit{Feel-Good Thompson Sampling} (FGTS) algorithm of
\citet{Zhang_22}.
For the $ 3
$-information ratio, we consider a mixture of the FGTS policy and an exploratory policy. The following lemma provides bounds on the surrogate information ratios of the SOIDS algorithm.

\begin{lemma}
\label{lemma:IR_bound}
	The $ 2 $- and $ 3 $-surrogate-information ratio of the SOIDS algorithm satisfy for any $ t\geq 0 $
\begin{equation}
\label{eq:2_IR_bound}
	\SIRs_t(\pIDS_t) \leq \SIRs_t(\pFGTS_t) \leq 2 d	
\end{equation}
and
\begin{equation}
\label{eq:3_IR_bound}
\SIRc_t(\pIDS_t) \leq 2 \SIRc_t(\pmix_t) \leq \frac{54s}{C_{\min}}.
\end{equation}
\end{lemma}

The explicit definition of both forerunner algorithms, as well as the proof of this lemma, are deferred to
Appendix~\ref{app:IR_bound}. Finally, it remains to pick the learning rate $ \lambda_t $. The following lemma describes the appropriate learning rate
for the data-poor and the data-rich regimes separately.

\begin{lemma}
\label{lemma:history_indep_learning_rates}
The choice of learning rate $ \lambdas_t =  \sqrt{\frac{3C_{t+1}}{128d(t+1)}} $
guarantees

\begin{equation*}
	\frac{C_T}{\lambdas_{T-1}} + \frac{32}{3}  \sum_{t=1}^{T} \lambdas_{t-1} \SIRs_{t}(\pIDS_t) \leq 16 \sqrt{
	\frac{2}{3} C_T d T}.
\end{equation*}
The choice of learning rate $ \lambdac_t = \frac{1}{4 \cdot 6^{\frac{1}{3}}} \left( \frac{ C_{t+1}\sqrt{C_{\min}}}{(t+1)\sqrt{s}} \right)^{\frac{2}{3}} $
guarantees
\begin{equation*}
	\frac{C_T}{\lambdac_{T-1}} + \frac{16}{3} c_3^{*} \sum_{t=1}^{T} \sqrt{3 \lambdac_{t-1} \SIRc_{t}(\pIDS_t)} \leq
	12 \cdot 6^{\frac{1}{3}} \left(\frac{s \cdot C_T }{C_{\min}}\right)^{\frac{1}{3}}	
	T^{\frac{2}{3}}.
\end{equation*}
\end{lemma}
The proof is deferred to Appendix~\ref{app:history_indep_learning_rates}. The expression for the constant $c_3^* \leq 2$ can be found in Appendix \ref{app:gen_am_gm} (cf. Lemma \ref{lemma:Gen_AM_GM}).
It remains to analyze what happens when the learning rate $ \lambda_t=\min(\frac{1}{2},\max(\lambdas_t, \lambdac_t) ) $
is chosen. We
defer this to Appendix~\ref{app:joint_learning_rate}.

\section{Experiments}

We aim to verify that, in both the data-rich and data-poor regimes simultaneously, the regret of SOIDS is comparable with the regret of existing algorithms that achieve near optimal worst-case regret in either the data-rich or the data-poor regime. Our baseline for the data-rich regime is the online-to-confidence-set (OTCS) method proposed by \citet{Abbas_P_S12}, which has worst case regret of the order $\sqrt{sdT}$. For a tougher comparison, we run this method with the confidence sets from Theorem 4.7 of \citet{clerico2025confidence}, which have much smaller constant factors than those used by \citet{Abbas_P_S12}. Our baseline for the data-poor regime is the Explore the Sparsity Then Commit (ESTC) algorithm proposed by \citet{Hao_L_W20}, which has worst-case regret of the order $(sT)^{2/3}$. For reference, we also compare with LinUCB \citet{abbasi2011improved}, which does not adapt to sparsity.

It is generally difficult to run the SOIDS algorithm exactly because the surrogate information ratio contains expectations w.r.t.\ the optimistic posterior. In our implementation of SOIDS, we use the empirical Bayesian sparse sampling procedure of \citet{Hao_L_D21a} to draw approximate samples from the optimistic posterior, and then approximate the surrogate information ratio via sample averages. We provide further details regarding the implementations of each method in Appendix \ref{sec:experiment_details}.

\begin{figure}
\centering
\includegraphics[width=\textwidth]{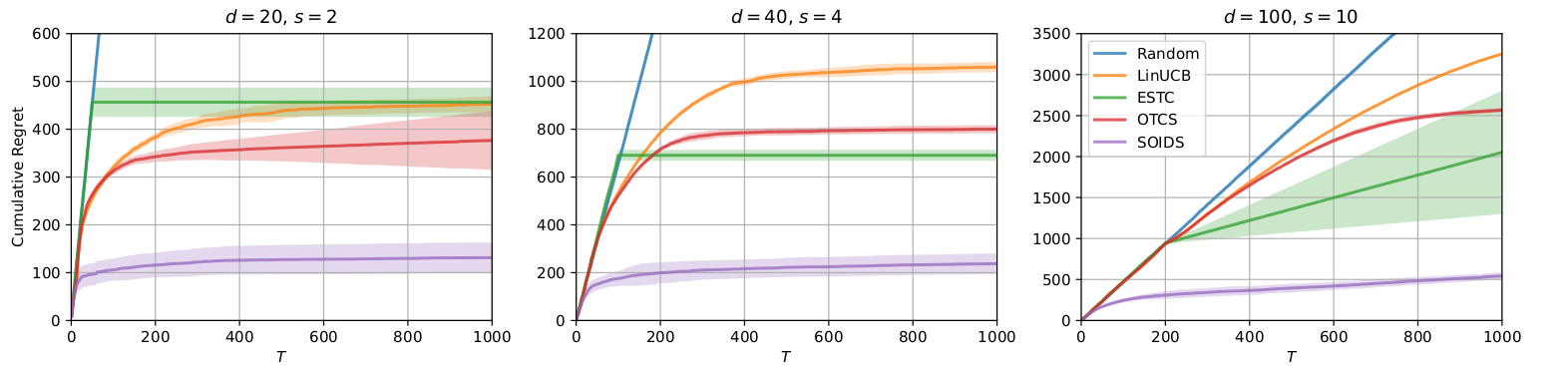}
\caption{Cumulative regret for $d = 20$ (left) $40$ (middle) and $100$ (right). We plot the mean $\pm$ standard deviation over 10 repetitions.}
\label{fig:regret}
\end{figure}

For each $d \in \{20, 40, 100\}$, $\theta_0$ is the $s$-sparse vector in $\mathbb{R}^d$, with $s = d/10$, in which first $s$ components are $10/s$ and the remaining components are zero. The action set consists of 200 random draws from the uniform distribution on $[-1, 1]^d$. The noise variance is 1 and we run each method 10 times. In Figure \ref{fig:regret}, we report the cumulative regret over $T=1000$ steps. As $d$ is varied from 20 to 100, we appear to transition from the data-rich regime to the data-poor regime: for $d = 20$, the OTCS method is the best performing baseline, whereas for $d=100$, ETCS is the best performing baseline. As our theoretical results would suggest, SOIDS performs well in both regimes.

\section{Conclusion}
There remain several interesting questions that our work leaves open for future 
research, such as the possibility of improving the logarithmic terms in the instance-dependent regret bound (as mentioned 
earlier in Section~\ref{sec:main}). We highlight another question below.

In our experiments, we have made use of an approximate implementation of OIDS adapted from \citet{Hao_L_D21a}. 
The initial success we have seen in our experiments suggests that this approximation might be viable in more 
challenging settings, and worthy of an attempt at a solid theoretical analysis. More broadly, the results indicate 
a potential advantage of IDS-style methods over DEC-inspired methods \citep{Foste_K_Q_R22,Kirsc_B_C_T_S23}. Indeed, we 
are not aware of any general methods for approximating the optimization problems that the E2D algorithm of 
\citet{Foste_K_Q_R22} requires to solve, in contrast to our results that indicate that IDS-inspired algorithms may very 
well be amenable to practical implementation. Whether the concrete approximation we used in our experiments is the best 
possible one or not remains to be seen.




\paragraph{Acknowledgements.}
The authors wish to thank Johannes Kirschner for thought-provoking discussions about information directed sampling before the preparation of this manuscript. This project has received funding from the European Research Council (ERC), under the European
Union’s Horizon 2020 research and innovation programme (Grant agreement No.~950180).

\clearpage
\bibliography{ref}

\begin{thebibliography}{42}
\providecommand{\natexlab}[1]{#1}
\providecommand{\url}[1]{\texttt{#1}}
\expandafter\ifx\csname urlstyle\endcsname\relax
  \providecommand{\doi}[1]{doi: #1}\else
  \providecommand{\doi}{doi: \begingroup \urlstyle{rm}\Url}\fi

\bibitem[Abbasi-Yadkori et~al.(2011)Abbasi-Yadkori, P{\'a}l, and Szepesv{\'a}ri]{abbasi2011improved}
Yasin Abbasi-Yadkori, D{\'a}vid P{\'a}l, and Csaba Szepesv{\'a}ri.
\newblock Improved algorithms for linear stochastic bandits.
\newblock \emph{Advances in neural information processing systems}, 24, 2011.

\bibitem[{Abbasi-Yadkori} et~al.(2012){Abbasi-Yadkori}, P{\'a}l, and Szepesv{\'a}ri]{Abbas_P_S12}
Yasin {Abbasi-Yadkori}, D{\'a}vid P{\'a}l, and Csaba Szepesv{\'a}ri.
\newblock Online-to-confidence-set conversions and application to sparse stochastic bandits.
\newblock 2012.

\bibitem[Abe and Long(1999)]{abe1999associative}
Naoki Abe and Philip~M Long.
\newblock Associative reinforcement learning using linear probabilistic concepts.
\newblock In \emph{ICML}, pages 3--11. Citeseer, 1999.

\bibitem[Abernethy et~al.(2008)Abernethy, Hazan, and Rakhlin]{Abern_H_R08}
Jacob~D. Abernethy, Elad Hazan, and Alexander Rakhlin.
\newblock Competing in the dark: {{An}} efficient algorithm for bandit linear optimization.
\newblock pages 263--274, 2008.
\newblock URL \url{http://colt2008.cs.helsinki.fi/papers/127-Abernethy.pdf}.

\bibitem[Alquier and Lounici(2011)]{alquier2011pac}
Pierre Alquier and Karim Lounici.
\newblock Pac-bayesian theorems for sparse regression estimation with exponential weights.
\newblock \emph{Electronic Journal of Statistics}, 5:\penalty0 127--145, 2011.

\bibitem[Auer(2002)]{auer2002using}
Peter Auer.
\newblock Using confidence bounds for exploitation-exploration trade-offs.
\newblock \emph{Journal of Machine Learning Research}, 3\penalty0 (Nov):\penalty0 397--422, 2002.

\bibitem[Bastani and Bayati(2020)]{bastani2020online}
Hamsa Bastani and Mohsen Bayati.
\newblock Online decision making with high-dimensional covariates.
\newblock \emph{Operations Research}, 68\penalty0 (1):\penalty0 276--294, 2020.

\bibitem[Boucheron et~al.(2013)Boucheron, Lugosi, and Massart]{Bouch_L_M13}
St{\'e}phane Boucheron, G{\'a}bor Lugosi, and Pascal Massart.
\newblock \emph{Concentration Inequalities - {{A}} Nonasymptotic Theory of Independence}.
\newblock 2013.
\newblock ISBN 978-0-19-953525-5.
\newblock \doi{10.1093/ACPROF:OSO/9780199535255.001.0001}.
\newblock URL \url{https://doi.org/10.1093/acprof:oso/9780199535255.001.0001}.

\bibitem[Bubeck and Sellke(2022)]{Bubec_S22}
S{\'e}bastien Bubeck and Mark Sellke.
\newblock First-{{Order Bayesian Regret Analysis}} of {{Thompson Sampling}}, 2022.
\newblock URL \url{http://arxiv.org/abs/1902.00681}.

\bibitem[Chakraborty et~al.(2023)Chakraborty, Roy, and Tewari]{chakraborty2023thompson}
Sunrit Chakraborty, Saptarshi Roy, and Ambuj Tewari.
\newblock Thompson sampling for high-dimensional sparse linear contextual bandits.
\newblock In \emph{International Conference on Machine Learning}, pages 3979--4008. PMLR, 2023.

\bibitem[Chu et~al.(2011)Chu, Li, Reyzin, and Schapire]{Chu_L_R_S11}
Wei Chu, Lihong Li, Lev Reyzin, and Robert~E. Schapire.
\newblock Contextual bandits with linear payoff functions.
\newblock volume~15 of \emph{{{JMLR}} Proceedings}, pages 208--214, 2011.
\newblock URL \url{http://proceedings.mlr.press/v15/chu11a/chu11a.pdf}.

\bibitem[Clerico et~al.(2025)Clerico, Flynn, Kotłowski, and Neu]{clerico2025confidence}
Eugenio Clerico, Hamish Flynn, Wojciech Kotłowski, and Gergely Neu.
\newblock Confidence sequences for generalized linear models via regret analysis, 2025.
\newblock URL \url{https://arxiv.org/abs/2504.16555}.

\bibitem[Dani et~al.(2008)Dani, Hayes, and Kakade]{dani2008stochastic}
Varsha Dani, Thomas~P Hayes, and Sham~M Kakade.
\newblock Stochastic linear optimization under bandit feedback.
\newblock In \emph{COLT}, volume~2, page~3, 2008.

\bibitem[Foster and Rakhlin(2020)]{Foste_R20}
Dylan~J. Foster and Alexander Rakhlin.
\newblock Beyond {{UCB}}: {{Optimal}} and {{Efficient Contextual Bandits}} with {{Regression Oracles}}, 2020.
\newblock URL \url{http://arxiv.org/abs/2002.04926}.

\bibitem[Foster et~al.(2022{\natexlab{a}})Foster, Golowich, Qian, Rakhlin, and Sekhari]{Foste_G_Q_R_S22}
Dylan~J. Foster, Noah Golowich, Jian Qian, Alexander Rakhlin, and Ayush Sekhari.
\newblock A {{Note}} on {{Model-Free Reinforcement Learning}} with the {{Decision-Estimation Coefficient}}, 2022{\natexlab{a}}.
\newblock URL \url{http://arxiv.org/abs/2211.14250}.

\bibitem[Foster et~al.(2022{\natexlab{b}})Foster, Kakade, Qian, and Rakhlin]{Foste_K_Q_R22}
Dylan~J. Foster, Sham~M. Kakade, Jian Qian, and Alexander Rakhlin.
\newblock The {{Statistical Complexity}} of {{Interactive Decision Making}}, 2022{\natexlab{b}}.
\newblock URL \url{http://arxiv.org/abs/2112.13487}.

\bibitem[Foster et~al.(2022{\natexlab{c}})Foster, Rakhlin, Sekhari, and Sridharan]{Foste_R_S_S22}
Dylan~J. Foster, Alexander Rakhlin, Ayush Sekhari, and Karthik Sridharan.
\newblock On the {{Complexity}} of {{Adversarial Decision Making}}, 2022{\natexlab{c}}.
\newblock URL \url{http://arxiv.org/abs/2206.13063}.

\bibitem[Gerchinovitz(2013)]{gerchinovitz2013sparsity}
S{\'e}bastien Gerchinovitz.
\newblock Sparsity regret bounds for individual sequences in online linear regression.
\newblock \emph{The Journal of Machine Learning Research}, 14\penalty0 (1):\penalty0 729--769, 2013.

\bibitem[Hao and Lattimore(2022)]{Hao_L22a}
Botao Hao and Tor Lattimore.
\newblock Regret bounds for information-directed reinforcement learning.
\newblock 2022.
\newblock URL \url{http://papers.nips.cc/paper_files/paper/2022/hash/b733cdd80ed2ae7e3156d8c33108c5d5-Abstract-Conference.html}.

\bibitem[Hao et~al.(2020)Hao, Lattimore, and Wang]{Hao_L_W20}
Botao Hao, Tor Lattimore, and Mengdi Wang.
\newblock High-dimensional sparse linear bandits.
\newblock 2020.

\bibitem[Hao et~al.(2021)Hao, Lattimore, and Deng]{Hao_L_D21a}
Botao Hao, Tor Lattimore, and Wei Deng.
\newblock Information directed sampling for sparse linear bandits.
\newblock 2021.

\bibitem[Hao et~al.(2022)Hao, Lattimore, and Qin]{Hao_L_Q22}
Botao Hao, Tor Lattimore, and Chao Qin.
\newblock Contextual {{Information-Directed Sampling}}, 2022.
\newblock URL \url{http://arxiv.org/abs/2205.10895}.

\bibitem[Hazan and Kale(2010)]{Hazan_K10}
Elad Hazan and Satyen Kale.
\newblock Extracting certainty from uncertainty: Regret bounded by variation in costs.
\newblock \emph{Machine Learning}, 80\penalty0 (2-3):\penalty0 165--188, 2010.
\newblock \doi{10.1007/S10994-010-5175-X}.
\newblock URL \url{https://doi.org/10.1007/s10994-010-5175-x}.

\bibitem[Jang et~al.(2022)Jang, Zhang, and Jun]{jang2022popart}
Kyoungseok Jang, Chicheng Zhang, and Kwang-Sung Jun.
\newblock Popart: Efficient sparse regression and experimental design for optimal sparse linear bandits.
\newblock \emph{Advances in Neural Information Processing Systems}, 35:\penalty0 2102--2114, 2022.

\bibitem[Kim and Paik(2019)]{kim2019doubly}
Gi-Soo Kim and Myunghee~Cho Paik.
\newblock Doubly-robust lasso bandit.
\newblock \emph{Advances in Neural Information Processing Systems}, 32, 2019.

\bibitem[Kirschner and Krause(2018)]{Kirsc_K18}
Johannes Kirschner and Andreas Krause.
\newblock Information {{Directed Sampling}} and {{Bandits}} with {{Heteroscedastic Noise}}, 2018.
\newblock URL \url{http://arxiv.org/abs/1801.09667}.

\bibitem[Kirschner et~al.(2020)Kirschner, Lattimore, and Krause]{Kirsc_L_K20}
Johannes Kirschner, Tor Lattimore, and Andreas Krause.
\newblock Information directed sampling for linear partial monitoring.
\newblock volume 125 of \emph{Proceedings of Machine Learning Research}, pages 2328--2369, 2020.
\newblock URL \url{http://proceedings.mlr.press/v125/kirschner20a.html}.

\bibitem[Kirschner et~al.(2021)Kirschner, Lattimore, Vernade, and Szepesv{\'a}ri]{Kirsc_L_V_S21}
Johannes Kirschner, Tor Lattimore, Claire Vernade, and Csaba Szepesv{\'a}ri.
\newblock Asymptotically optimal information-directed sampling.
\newblock volume 134 of \emph{Proceedings of Machine Learning Research}, pages 2777--2821, 2021.
\newblock URL \url{http://proceedings.mlr.press/v134/kirschner21a.html}.

\bibitem[Kirschner et~al.(2023)Kirschner, Bakhtiari, Chandak, Tkachuk, and Szepesv{\'a}ri]{Kirsc_B_C_T_S23}
Johannes Kirschner, Seyed~Alireza Bakhtiari, Kushagra Chandak, Volodymyr Tkachuk, and Csaba Szepesv{\'a}ri.
\newblock Regret minimization via saddle point optimization.
\newblock 2023.
\newblock URL \url{http://papers.nips.cc/paper_files/paper/2023/hash/6eaf8c729af4fbeb18006dc2e6a41d9b-Abstract-Conference.html}.

\bibitem[Lattimore and Gy{\"o}rgy(2021)]{Latti_G21}
Tor Lattimore and Andr{\'a}s Gy{\"o}rgy.
\newblock Mirror {{Descent}} and the {{Information Ratio}}.
\newblock volume 134 of \emph{Proceedings of {{Machine Learning Research}}}, pages 2965--2992, 2021.
\newblock URL \url{http://proceedings.mlr.press/v134/lattimore21b.html}.

\bibitem[Lattimore and Szepesv{\'a}ri(2020)]{lattimore2020bandit}
Tor Lattimore and Csaba Szepesv{\'a}ri.
\newblock \emph{Bandit algorithms}.
\newblock Cambridge University Press, 2020.

\bibitem[Neu et~al.(2022)Neu, Olkhovskaya, Papini, and Schwartz]{Neu_O_P_S22}
Gergely Neu, Julia Olkhovskaya, Matteo Papini, and Ludovic Schwartz.
\newblock Lifting the {{Information Ratio}}: {{An Information-Theoretic Analysis}} of {{Thompson Sampling}} for {{Contextual Bandits}}, 2022.
\newblock URL \url{http://arxiv.org/abs/2205.13924}.

\bibitem[Neu et~al.(2024)Neu, Papini, and Schwartz]{Neu_P_S24}
Gergely Neu, Matteo Papini, and Ludovic Schwartz.
\newblock Optimistic information directed sampling.
\newblock 2024.

\bibitem[Oh et~al.(2021)Oh, Iyengar, and Zeevi]{oh2021sparsity}
Min-hwan Oh, Garud Iyengar, and Assaf Zeevi.
\newblock Sparsity-agnostic lasso bandit.
\newblock In \emph{International Conference on Machine Learning}, pages 8271--8280. PMLR, 2021.

\bibitem[Orabona(2019)]{Orabo_19}
Francesco Orabona.
\newblock A modern introduction to online learning.
\newblock \emph{CoRR}, abs/1912.13213, 2019.
\newblock URL \url{http://arxiv.org/abs/1912.13213}.

\bibitem[Rusmevichientong and Tsitsiklis(2010)]{rusmevichientong2010linearly}
Paat Rusmevichientong and John~N Tsitsiklis.
\newblock Linearly parameterized bandits.
\newblock \emph{Mathematics of Operations Research}, 35\penalty0 (2):\penalty0 395--411, 2010.

\bibitem[Russo and Roy(2016)]{Russo_R16}
Daniel Russo and Benjamin~Van Roy.
\newblock An information-theoretic analysis of thompson sampling.
\newblock \emph{Journal of Machine Learning Research}, 17:\penalty0 68:1--68:30, 2016.
\newblock URL \url{https://jmlr.org/papers/v17/14-087.html}.

\bibitem[Russo and Van~Roy(2017)]{Russo_V17}
Daniel Russo and Benjamin Van~Roy.
\newblock Learning to {{Optimize}} via {{Information-Directed Sampling}}, 2017.
\newblock URL \url{http://arxiv.org/abs/1403.5556}.

\bibitem[Valko et~al.(2014)Valko, Munos, Kveton, and Koc{\'a}k]{Valko_M_K_K14}
Michal Valko, R{\'e}mi Munos, Branislav Kveton, and Tom{\'a}{\v s} Koc{\'a}k.
\newblock Spectral bandits for smooth graph functions.
\newblock volume~32 of \emph{{{JMLR}} Workshop and Conference Proceedings}, pages 46--54, 2014.
\newblock URL \url{http://proceedings.mlr.press/v32/valko14.html}.

\bibitem[Wainwright(2019)]{Wainw_19}
M.J. Wainwright.
\newblock \emph{High-Dimensional Statistics: A Non-Asymptotic Viewpoint}.
\newblock Cambridge Series in Statistical and Probabilistic Mathematics. 2019.
\newblock ISBN 978-1-108-49802-9.
\newblock URL \url{https://books.google.es/books?id=8C8nuQEACAAJ}.

\bibitem[Wang et~al.(2018)Wang, Wei, and Yao]{wang2018minimax}
Xue Wang, Mingcheng Wei, and Tao Yao.
\newblock Minimax concave penalized multi-armed bandit model with high-dimensional covariates.
\newblock In \emph{International Conference on Machine Learning}, pages 5200--5208. PMLR, 2018.

\bibitem[Zhang(2022)]{Zhang_22}
Tong Zhang.
\newblock Feel-good thompson sampling for contextual bandits and reinforcement learning.
\newblock \emph{SIAM Journal on Mathematics of Data Science}, 4\penalty0 (2):\penalty0 834--857, 2022.
\newblock \doi{10.1137/21M140924X}.
\newblock URL \url{https://doi.org/10.1137/21m140924x}.

\end{thebibliography}
\appendix
\section{Related work}

The first algorithms and regret bounds for sparse linear bandits were designed for the data-rich regime. \citet{Abbas_P_S12} developed an online-to-confidence-set conversion for linear models, which converts any algorithm for online linear regression into a linear bandit algorithm whose regret depends on the regret of the online regression algorithm. When the SeqSEW algorithm \citep{gerchinovitz2013sparsity} is used in this conversion, the result is a sparse linear bandit algorithm with a regret bound of the order $\OO(\sqrt{sdT})$ (ignoring logarithmic factors). \citet{lattimore2020bandit} established a matching lower bound for the data-rich regime, showing that this rate cannot be improved.

More recently, several works have studied the data-poor regime, in which the dimension $d$ is much larger than the number of rounds $T$. \citet{Hao_L_W20} showed that an explore-then-commit algorithm satisfies a regret bound of the order $O((sT)^{2/3}C_{\min}^{-2/3})$, and established a lower bound of order $\Omega(\min(s^{1/3}T^{2/3}C_{\min}^{-1/3}, \sqrt{dT})$. Subsequently, \citet{jang2022popart} proposed the PopArt estimator for sparse linear regression, and showed that an explore-then-commit algorithm that uses this estimator achieves a regret bound of the order $\OO(s^{2/3}T^{2/3}H_{\star}^{2/3})$, where $H_{\star}$ is another problem-dependent quantity that satisfies $H_{\star}^2 \leq C_{\min}^{-1}$. In addition, \citet{jang2022popart} established a lower bound of order $\Omega(s^{2/3}T^{2/3}C_{\min}^{-1/3})$, showing that the optimal rate for the data-poor regime is $s^{2/3}T^{2/3}$. \citet{Hao_L_D21a} showed that sparse IDS has a Bayesian regret bound that is optimal for both regimes.

A number of works have considered sparse contextual linear bandits, in which the action set $\A$ changes in each round $t$. In the case where the actions sets are chosen by an adaptive adversary, the upper and lower bounds of the order $\sqrt{sdT}$ by \citet{Abbas_P_S12} and \citet{lattimore2020bandit} respectively still hold. Under the assumption that the action sets are generated randomly, and such that either a uniform or greedy policy is (with high probability) exploratory, several methods have been shown to achieve nearly dimension-free regret bounds \citet{bastani2020online, wang2018minimax, kim2019doubly, oh2021sparsity, chakraborty2023thompson}.

The concept of balancing instantaneous regret and information gain through the information ratio was first introduced by \citet{Russo_R16} in the context of analyzing Thompson Sampling. Building upon this, the Information-Directed Sampling (IDS) algorithm was proposed by \citet{Russo_V17} to directly minimize the information ratio, thereby optimizing the trade-off between regret and information gain. These foundational ideas have since been extended and applied to a variety of settings including bandits \citep{Bubec_S22}, contextual bandits \citep{Neu_O_P_S22, Hao_L_Q22}, reinforcement learning \citep{Hao_L22a}, and sparse linear bandits \citep{Hao_L_D21a}. However, these works are primarily situated in the Bayesian framework and focus on Bayesian regret bounds that hold only in expectation with respect to the prior distribution.

A key challenge in extending these methods to the frequentist setting lies in estimating the instantaneous regret and
define a meaningful notion of information gain. Both of those things are naturally possible in Bayesian analysis but difficult when the true model is unknown. Moreover, Bayesian posteriors may inadequately represent model uncertainty from a frequentist perspective. We highlight three strands of research that have attempted to address this challenge:

Confidence-set based information ratio approaches:
Works such as \citet{Kirsc_K18}, \citet{Kirsc_L_K20}, and \citet{Kirsc_L_V_S21} extend the notion of the information
ratio to frequentist settings by constructing high-probability confidence sets for the instantaneous regret and
information gain. These results are mostly limited to setting with some linear structure. 

Distributionally robust and worst-case information-regret trade-offs:
The Decision-to-Estimation-Coeffiecient(DEC) line of work of \citep{Foste_K_Q_R22, Foste_R20, Foste_R_S_S22, Foste_G_Q_R_S22,
Kirsc_B_C_T_S23} explores the frequentist setting by analyzing worst-case trade-offs between regret and information
gain. One limitation is that the DEC is an inherently worst-case measure of comlexity. Moreover, algorithms based on the
DEC usually require solving complex min-max optimization problems at each time step, making their practical implementation challenging and unclear.

Optimistic posterior approaches for frequentist guarantees:
The approach most closely related to our work modifies the Bayesian posterior to provide frequentist guarantees.
Introduced by \citet{Zhang_22}, the optimistic posterior is a modification of the Bayesian posterior which  enables
frequentist regret bounds for a variant of Thompson Sampling. Subsequently, \citet{Neu_P_S24} studied the optimistic
posterior framework in greater depth, defining a frequentist analog of the information ratio to extend IDS to
frequentist settings. A notable limitation of these works is their restriction to constant learning rates in the
optimistic posterior, which limits adaptivity, an issue that we address in this paper.

\section{Analysis of the Optimistic posterior}
This section provides further details about the prior underlying the optimistic posterior and guarantees on the 
posterior updates.
\label{app:Optimistic_posterior}
\subsection{Follow the regularized leader analysis}
The main step in our analysis of the optimistic posterior is to leverage the follow the regularized leader formulation
of our optimistic posterior update 
\begin{equation*}
	Q_{t+1}^{+} = \argmin_{P\in \Delta(\Theta)} \siprod{P}{\eta \Llik_t + \lambda_t \Lest_t} + \Phi(P)\,.
\end{equation*}
\subsubsection{Proof of Lemma~\ref{lemma:FTRL_decomposition}}
\label{app:FTRL_decomposition}
As is usual in the analysis of the follow the regularized leader algorithm, we introduce the Fenchel conjugate $\Phi^{*}: \real^{\Theta} \rightarrow \real$ of the 
regularization function $\Phi = \DKL{\cdot}{Q_1^+}$, which takes values $\Phi^*(L) = \sup_{P \in \Delta(\Theta)} \ev{\siprod{P}{L} - \Phi(P)}$.
The Fenchel--Young inequality guarantees that for any $ P \in \Delta(\Theta) $ and $ L \in \real^{\Theta} $,
\begin{equation*}
	\siprod{P}{L} \leq  \Phi(P) + \Phi^{*}(L)\,.
\end{equation*}
We introduce the maximum likelihood estimator $ \theta_t = \argmin_{\theta \in \Theta} \Llik_t(\theta) $ and the function
\begin{equation*}
L(\cdot) = -\eta (\Llik_T(\cdot )- \Llik_T(\theta_T)) -
\lambda_T \Lest_T(\cdot )\,.
\end{equation*}
Since $ \lambda_T $ is never used by the algorithm, we can assume that $ \lambda_T
= \lambda_{T-1}$. The role of the maximum likelihood estimator is to ensure that the term $ \Llik_t(\theta) -
\Llik_t(\theta_t)$ is always non-negative. The Fenchel--Young inequality tells us that
\begin{equation*}
	\eta \pa{\Llik_T(\theta_T) - \biprod{P}{\Llik_T}} - \lambda_T \biprod{P}{\Lest_T} \leq \Phi(P) + \Phi^{*}\pa{- 
\eta(\Llik_T(\cdot ) -
	\Llik_T(\theta_T)) -
	\lambda_T \Lest_T(\cdot )}\,.
\end{equation*}
Noticing that $\biprod{P}{\Lest_T} = -  \sum_{t=1}^T \Delta(P,A_t)$ and rearranging the terms concludes the proof.

\subsubsection{Proof of Lemma~\ref{lemma:pre_telescoping_phi_s}}
\label{app:pre_telescoping_phi_s}
We start by rewriting the potential function in the form of the following telescopic sum:
\begin{align*}
	&\phantom{=}\frac{\Phi^{*}(- \eta (\Llik_T(\cdot )-\Llik_T(\theta_T)) - \lambda_T \Lest_T(\cdot ) )}{\lambda_T}\\
	&=\sum_{t=1}^{T}\frac{\Phi^{*}(- \eta (\Llik_t(\cdot )-\Llik_t(\theta_t)) - \lambda_t \Lest_t(\cdot ) )}{\lambda_t}
	-\frac{\Phi^{*}(- \eta (\Llik_{t-1}(\cdot )-\Llik_{t-1}(\theta_{t-1})) - \lambda_{t-1} \Lest_{t-1}(\cdot )
	)}{\lambda_{t-1}}.
\end{align*}
In the usual follow-the-regularized-leader analysis, we use the fact that $ \lambda \mapsto \frac{\phi^{*}(\lambda 
L)}{\lambda} $ is non-decreasing for any $ L \in \real^{\Theta} $. Here however, only one component of the linear loss is scaled
by $ \lambda_t $, and so the standard FTRL analysis fails. Crucially, because we introduced the maximum 
likelihood estimator $ \theta_t $, we
have that $ \Llik_t(\cdot ) - \Llik_t(\theta_t) \geq 0$. We can therefore use the following lemma that guarantees that
a scaled and shifted dual is monotonic.

\begin{lemma}
\label{lemma:increasing_dual}
Let $ \Phi \geq 0 $ and $ \Phi^{*} $ be a convex function and its dual as defined previously. For $ L_1, L_2 \in \real^{\Theta} $ such that $
L_1 \geq 0 $, the mapping $ \lambda \mapsto \frac{\Phi^{*}(-L_1 + \lambda L_2)}{\lambda} $ is
non-decreasing on $\real^{+*}$.
\end{lemma}
\begin{proof}
By definition, we have 
\begin{align*}
	\frac{\Phi^{*}(- L_1 + \lambda L_2)}{\lambda} &= \frac{\sup_{P \in \Delta(\Theta)} \siprod{P}{-L_1 + \lambda L_2} -
	\Phi(P)}{\lambda} \\
						      &=\sup_{P \in \Delta(\Theta)} \siprod{P}{L_2} -
						      \frac{\siprod{P}{L_1} + \Phi(P)}{\lambda}.
\end{align*}
For any $ P\in \Delta (\Theta) $, we have that $ \Phi(P) + \siprod{P}{L_1} \geq 0 $ and the term inside the supremum is
non-decreasing with
respect to lambda. Since the supremum of non-decreasing functions is also non-decreasing, this concludes the proof.
\end{proof}

Applying the previous lemma, we upper bound the previous sum by replacing each $ \lambda_t $ factor by $ \lambda_{t-1} 
$ (using the convention $\lambda_0 = 1/2$), and then we replace the maximum likelihood estimator $ \theta_t $ 
by $ \theta_0 $ inside $ \Phi^{*} $ to obtain
\begin{align*}
	&\phantom{=}\sum_{t=1}^{T}\frac{\Phi^{*}(- \eta (\Llik_t(\cdot )-\Llik_t(\theta_t)) - \lambda_t \Lest_t(\cdot ) )}{\lambda_t}
	-\frac{\Phi^{*}(- \eta (\Llik_{t-1}(\cdot )-\Llik_{t-1}(\theta_{t-1})) - \lambda_{t-1} \Lest_{t-1}(\cdot )
	)}{\lambda_{t-1}}\\
	&\leq \sum_{t=1}^{T}\frac{\Phi^{*}(- \eta (\Llik_t(\cdot )-\Llik_t(\theta_{t})) - \lambda_{t-1} \Lest_t(\cdot )
	)}{\lambda_{t-1}}
	-\frac{\Phi^{*}(- \eta (\Llik_{t-1}(\cdot )-\Llik_{t-1}(\theta_{t-1})) - \lambda_{t-1} \Lest_{t-1}(\cdot )
	)}{\lambda_{t-1}} \\
	&= \sum_{t=1}^{T}\frac{\Phi^{*}(- \eta (\Llik_t(\cdot )-\Llik_t(\theta_{0})) - \lambda_{t-1} \Lest_t(\cdot )
	)}{\lambda_{t-1}}
	-\frac{\Phi^{*}(- \eta (\Llik_{t-1}(\cdot )-\Llik_{t-1}(\theta_0)) - \lambda_{t-1} \Lest_{t-1}(\cdot )
	)}{\lambda_{t-1}}\\
	&+ \frac{\eta}{\lambda_{t-1}}(\Llik_t(\theta_t) - \Llik_t(\theta_0) + \Llik_{t-1}(\theta_0) -
	\Llik_{t-1}(\theta_{t-1})).
\end{align*}
It remains to bound the difference of the negative log likelihood of the true parameter and the maximum likelihood
estimator. This is done via the following result (whose proof we relegate to appendix~\ref{app:MLE_bound_complete}).
\begin{lemma}
\label{lemma:MLE_bound_complete}
For any $ t \geq 1 $, we have
\begin{equation}
\label{eq:MLE_bound_complete}
	0 \leq \EE{\Llik_t(\theta_0) - \Llik_t( \theta_t)} \leq \inf_{\rho} \left\{ 2 \rho t+ s \log \frac{ed(1 + 2/
	\rho)}{s} \right\} \leq 6 + s \log\frac{edt}{s}
\end{equation}
\end{lemma}
Using this lemma, we can further bound the previously considered expression as the following telescopic sum:
\begin{align*}
 &\EE{\sum_{t=1}^{T}\frac{\eta}{\lambda_{t-1}}(\Llik_t(\theta_t) - \Llik_t(\theta_0) + \Llik_{t-1}(\theta_0) -
	\Llik_{t-1}(\theta_{t-1})) + \frac{\eta}{\lambda_T}(\Llik_T(\theta_0) - \Llik_T(\theta_T))}  \\	
 &\qquad= \EE{\sum_{t=1}^{T} \frac{\eta}{\lambda_{t-1}}(\Llik_t(\theta_t) - \Llik_t(\theta_0)) - \sum_{t=1}^{T}
 \frac{\eta}{\lambda_{t}}(\Llik_t(\theta_t) - \Llik_t(\theta_0))}\\
 &\qquad\leq \eta\cdot \sum_{t=1}^{T}\EE{(\Llik_t(\theta_0)- \Llik_t(\theta_t))} \left( \frac{1}{\lambda_{t}} -
	 \frac{1}{\lambda_{t-1}}
 \right)\\
&\qquad\leq \frac{\eta (6 + s \log \frac{edT}{s})}{\lambda_T}.
\end{align*}
Here, the first inequality comes from the non-negativity of $ \Llik_t(\theta_0) - \Llik_t(\theta_t) $ by definition of
$ \theta_t $ and the second one is from Lemma~\ref{lemma:MLE_bound_complete} just above and a telescoping argument. 
Finally we obtain the claim of Lemma~\ref{lemma:pre_telescoping_phi_s}.
\subsubsection{Controlling the losses separately}
The focus of this section is to understand how to control $ \Phi^{*}(-L) $ where $ L $ is either the negative-likelihood 
loss or the estimation-error loss. We start by analyzing the negative-likelihood loss. As was done in 
\citet*{Neu_P_S24}, we
will relate the negative-likelihood loss to the surrogate information gain.

For this analysis, we define the \emph{true information gain} as 
\begin{equation}
\label{eq:TIG}
	\IG_t(\pi) = \frac{1}{2} \sum_{a \in \A} \pi(a) \int_{\Theta}(\siprod{\theta - \theta_0}{a} )^2 \, \dd Q_t^{+}(\theta),
\end{equation}
and note that, by linearity reward function, the surrogate information gain is always smaller than the true 
information gain. This is stated formally below.
\begin{proposition}
\label{prop:SIG_TIG}
For any policy $ \pi \in \Delta(\A) $ and any $ t\geq 1 $ we have that 
\begin{equation}
\label{eq:SIG_TIG}
	\SIG_t(\pi) \leq  \IG_t(\pi)
\end{equation}
\end{proposition}
The proof is provided in Appendix~\ref{app:SIG_TIG}. This result can then be used to relate the surrogate and the true 
information gain to the negative-likelihood loss. This result and its proof are identical to the proof of Lemma~17 in 
\citet*{Neu_P_S24}.
\begin{lemma}
\label{lemma:likelihood_IG_bound}
Assume that the noise $ \epsilon_t $ is conditionnally 1-sub-Gaussian, then for any $ t \geq 1, \eta, \alpha \geq 0 $
such that $ \gamma =  \frac{\eta \alpha}{2} \left( 1 - \eta \alpha \right) > 0$, the following inequality 
holds
\begin{align}
\label{eq:likelihood_IG_bound}
\EE{ \log \int_{\Theta} \left( \frac{p(Y_t|\theta, A_t)}{p(Y_t|\theta_0, A_t)} \right)^{\eta \alpha} \, \mathrm{d}Q_t^{+}(\theta)}
&\leq -2\gamma(1 - 2 \gamma)\EE{\IG_t(\pi_t)}\\
&\leq -2\gamma(1 - 2 \gamma)\EE{\SIG_t(\pi_t)}.
\end{align}
In particular, the constant $ 2\gamma(1-2\gamma) $ can be maximized to the value $ \frac{3}{16} $ by the choice $ \eta
\alpha = \frac{1}{2} $.
	
\end{lemma}
\begin{proof}
By the tower rule of expectation and Jensen's inequality applied to the logarithm, we have 
\begin{align*}
	&\EE{\log\int_\Theta \left( \frac{p(Y_t|\theta, A_t)}{p(Y_t|\theta_0, A_t)} \right)^{\eta \alpha}\mathrm{d}Q_t^{+}(\theta)}=
	\EE{\EEc{\log\int_\Theta \left( \frac{p(Y_t|\theta, A_t)}{p(Y_t|\theta_0, A_t)} \right)^{\eta \alpha}\mathrm{d}Q_t^{+}(\theta)}{\F_{t-1},
	A_t}}
	\\
	&\qquad\qquad \leq \EE{\log\EEc{\int_\Theta \left( \frac{p(Y_t|\theta, A_t)}{p(Y_t|\theta_0, A_t)} 
\right)^{\eta \alpha}\mathrm{d}Q_t^{+}(\theta)}{\F_{t-1},
	A_t}}\\ 
	&\qquad\qquad= \EE{\log \int_\Theta \EEc{\exp \left( - \eta \alpha \left( \frac{(Y_t - \siprod{\theta}{A_t})^2}{2} 
-
	\frac{(Y_t - \siprod{\theta_0}{A_t})^2}{2} \right) \right)}{\F_{t-1}, A_t}\mathrm{d}Q_t^{+}(\theta)}.
\end{align*}
Now, we fix some $ \theta \in \Theta $ and to simplify the notation, we let $ r_0 = \siprod{\theta_0}{A_t} $ and $ r=
\siprod{\theta}{A_t} $.
Using some elementary manipulations and the conditional sub-Gaussian tail behaviour of $ \epsilon_t $ and $ Y_t = r_0 + \epsilon_t $
which implies that for any $ (\F_{t-1},
A_t)$-measurable $\zeta_t
$, $ \EEc{\exp \left( Y_t \zeta_t \right)}{\F_{t-1}, A_t} = \exp(r_0 \zeta_t) \EEc{\exp \left( \epsilon_t \zeta_t
\right)}{\F_{t-1}, A_t} \leq \exp(r_0 \zeta_t) \exp\left(\frac{\zeta_t^2}{2}\right) $, we have
\begin{align*}
	&\EEc{\exp \left( - \eta \alpha \left(  \frac{(Y_t - r)^2}{2} - \frac{(Y_t - r_0)^2}{2}\right)\right)}{\F_{t-1},A_t}\\
	&\qquad\qquad= \EEc{\exp \left( - \frac{\eta \alpha}{2}   (2Y_t - r - r_0)(r_0 - r)  \right)}{\F_{t-1},A_t}\\
	&\qquad\qquad= \exp \left(\eta \alpha \frac{r_0^2-r^2}{2} \right) \EEc{\exp \left( \eta \alpha Y_t(r- 
r_0)\right)}{\F_{t-1}, A_t}\\
	&\qquad\qquad\leq \exp \left(\eta \alpha \frac{r_0^2-r^2}{2} \right) \cdot \exp \left( \eta \alpha r_0(r- r_0) 
\right) \exp
	\left(\frac{\eta^2 \alpha^2}{2} (r-r_0)^2) \right)\\
	&\qquad\qquad= \exp\left(-(r-r_0)^2 \cdot \frac{\eta \alpha}{2} \left( 1 - \eta \alpha \right)\right).
\end{align*}
Further, defining $ \gamma = \frac{\eta \alpha}{2} \left( 1 - \eta \alpha \right) $, we have
\begin{align*}
	\phantom{=}&\EEc{\exp \left( - \eta \alpha \left(  \frac{(Y_t - r)^2}{2} - \frac{(Y_t - r_0)^2}{2}\right)\right)}{\F_{t-1},A_t}\\
		   &\leq \exp(-\gamma(r- r_0)^2)\\
		   &\leq 1 - \gamma(r- r_0)^2 + \frac{\gamma^2}{2} (r- r_0)^4\\
		   &\leq 1 - \gamma(r- r_0)^2 + 2 \gamma^2(r- r_0)^2\\
		   &\leq 1 - \gamma(1- 2 \gamma) (r- r_0)^2.
\end{align*}
Here, we used the elementary inequality $ \exp(x) \leq 1 + x + \frac{x^2}{2} $ for $ x\leq 0 $ and then used $ |r - r_0|
\leq 2$.
Finally, using that $ \log x \leq x-1 $ for any $ x>0 $, and taking the integral over $ \Theta $, we get that
\begin{align*}
	\EE{\log\int_\Theta \left( \frac{p(Y_t|\theta, A_t)}{p(Y_t|\theta_0, A_t)} \right)^{\eta \alpha}\mathrm{d}Q_t^{+}(\theta)} &\leq -
	\gamma(1- 2 \gamma)
	\EE{\sum_{a \in \A} \pi_t(A)\int_\Theta (\siprod{\theta- \theta_0}{a})^2\mathrm{d}Q_t^{+}(\theta)}\\
	&= - 2\gamma (1 - 2\gamma)\EE{\IG_t(\pi_t)}.
\end{align*}
Rearranging and combining the result with Proposition~\ref{prop:SIG_TIG} yields the claim of the lemma.
\end{proof}

We now turn our focus to the estimation error loss and relate it to the surrogate regret through the following lemma, 
whose proof is a straightforward application of Lemma~\ref{lemma:data_dep_hoeffding}.
\begin{lemma}
\label{lemma:gap_bound}
For any $ t\geq 1 , \beta>1$, if $ \beta \lambda_{t-1} \leq 1$, we have
\begin{equation}
\label{eq:gap_bound}
\EE{\frac{1}{\beta \lambda_{t-1}}\log \int_{\Theta}\exp(\beta \lambda_{t-1}\Delta(A_t, \theta))\, \mathrm{d}Q_t^{+}(\theta)} \leq
\EE{2 \Deltah_t(\pi_t)}.
\end{equation}

\end{lemma}

\subsubsection{Separation of the two losses: proof of Lemma~\ref{lemma:telescoping_phi_s}}
\label{app:telescoping_phi_s}
We make use of the fact that the Fenchel dual of $\Phi$ can be explicitly written as $ \Phi^{*}(L) = \log 
\int_{\Theta} \exp (L(\theta)) \, \mathrm{d}Q_1^{+}(\theta) $ . As a
result, we have
\begin{align*}	
\nonumber&\phantom{=} \EE{\sum_{t=1}^{T}\frac{\Phi^{*}(- \eta (\Llik_t(\cdot )-\Llik_t(\theta_{0})) - \lambda_{t-1} \Lest_t(\cdot )
	)}{\lambda_{t-1}}
	-\frac{\Phi^{*}(- \eta (\Llik_{t-1}(\cdot )-\Llik_{t-1}(\theta_0)) - \lambda_{t-1} \Lest_{t-1}(\cdot )
	)}{\lambda_{t-1}}}\\
	 &=\EE{\sum_{t=1}^{T} \frac{1}{\lambda_{t-1}}\log \frac{\int_{\Theta} \left(
	 \frac{p(Y_t|\theta,a_t)}{p(Y_t|\theta_0, A_t)} \right)^\eta \exp \left( \lambda_{t-1} \Delta(A_t, \theta)\right) \exp
\left(- \eta \Llik_{t-1}(\theta) - \lambda_{t-1} \Lest_{t-1}(\theta)\right)\, \mathrm{d}Q_1^{+} (\theta)}{\int_{\Theta}\exp
\left(- \eta \Llik_{t-1}(\theta) - \lambda_{t-1} \Lest_{t-1}(\theta)\right)\, \mathrm{d}Q_1^{+} (\theta)}}\\
	 &= \EE{ \sum_{t=1}^{T} \frac{1}{\lambda_{t-1}}\log \int_{\Theta} \left(
	 \frac{p(Y_t|\theta,A_t)}{p(Y_t|\theta_0, A_t)} \right)^\eta \exp \left( \lambda_{t-1} \Delta(A_t,
\theta)\right)\, \mathrm{d}Q_t^{+}(\theta)}\\
	 &\leq \EE{\sum_{t=1}^T \frac{1}{\alpha\lambda_{t-1}} \log \int_{\Theta}\left(
	 \frac{p(Y_t|\theta,A_t)}{p(Y_t|\theta_0, A_t)} \right)^{\eta \alpha} + \frac{1}{\beta
	 \lambda_{t-1}}\log\int_{\Theta}  \exp \left( \beta\lambda_{t-1} \Delta(A_t,
\theta)\right)\, \mathrm{d}Q_t^{+}(\theta)},
\end{align*}
where the last equality is by definition of the optimistic posterior and the last inequality follows from using 
Hölder's inequality with the two real numbers $ \alpha, \beta >1 $ that satisfy $\frac{1}{\alpha} + \frac{1}{\beta} =1
$.
Combining Lemma~\ref{lemma:likelihood_IG_bound} and Lemma~\ref{lemma:gap_bound} with the choice $ \alpha = \beta = 2 $,
the fact that $ \eta = \frac{1}{4} $
and the last inequality yields the
claim of the lemma. \qed
\subsection{Choice of the prior and comparator distribution: proof of Lemma~\ref{lemma:prior_comparator_bound}}
\label{app:prior_comparator_bound}
In order to construct the prior $Q_1^{+}$ and the comparator $P$ for the regret analysis, we need to take into account two 
criteria:  that $ \DKL{P}{Q_1^{+}} $ be controlled and that $|\siprod{P}{L} - L(\theta_0)| $ be small. The 
comparator will be a function of the unknown parameter $\theta_0$, and thus we denote it by $P_{\theta_0}$.

As for the 
prior, it should take into account the sparsity level of the unknown $\theta_0$, but should have no access to its 
support. We first design a distribution $\Pi$ over the set of all subsets of $[d] = \{1, \dots, d\}$, which have 
cardinality at most $s$. We choose the distribution such that: a) the probability assigned to each subset depends only 
on its cardinality; b) the probability assigned to the set of all subsets of size $k$ is proportional to $2^{-k}$, where 
$1 \leq k \leq s$. In other words, we prefer smaller subsets and have no preference over which indices in $[d]$ are 
included. The distribution that satisfies these requirements is\looseness=-1
\begin{equation}
\Pi(S) = \frac{2^{-|S|}}{\binom{d}{|S|}\sum_{k=1}^{s}2^{-k}}\,.\label{eqn:subset_dist}
\end{equation}

For $S = \emptyset$, we set $\Pi(S) = 0$. Doing so only complicates matters if the support of $\theta_0$ is empty 
(i.e., $\theta_0 = 0$). However, in this case, the reward function is $0$ everywhere, which means any algorithm would 
have $0$ regret. We therefore continue under the assumption that $\theta_0 \neq 0$. The most important property of this 
distribution, which we will use later, is that for any subset $S$ of cardinality $s$, $\log(1/\Pi(S)) \leq 
s\log(2ed/s)$. For each subset $S$, we define $Q_S$ to be the uniform distribution on $\Theta_S$. The prior is defined 
to be
\begin{equation*}
Q_1^{+} = \sum_{S \subset [d]: |S|\leq s}\Pi(S)Q_S\,.
\end{equation*}

%
%

As for the comparator distribution $ P_{\theta_0} $, we would ideally like to take a Dirac measure on $\theta_0$, but 
this would make the KL divergence appearing in the bound blow up. Thus, we pick a comparator $P_{\theta_0}$ which dilutes its mass 
around $\theta_0$. For any $\bar{\theta} \in \Theta$, with support $\bar{S}$, and any $\epsilon \in (0,1)$, we define 
the set $(1-\epsilon)\bar{\theta} +
 \epsilon\Theta_{\bar{S}} = \{(1-\epsilon)\bar{\theta} + \epsilon\theta^{\prime}: \theta^{\prime} \in \Theta_{\bar{S}}\} \subset
 \Theta_{\bar{S}}$. We choose $P_{\theta_0}$ to be the uniform distribution on $(1-\epsilon)\theta_0 + \epsilon\Theta_{S_0}$, where $S_0$ is the support of $\theta_0$. 
We now bound $\Phi(P_{\theta_0}) = \DKL{P_{\theta_0}}{Q_1^{+}}$ for this choice of $P_{\theta_0}$ in the following lemma, from which the claim 
of Lemma~\ref{lemma:prior_comparator_bound} then directly follows.

\begin{lemma}
For any $\bar{\theta} \in \Theta$, let $\bar{S}$ denote its support, and let $|\bar{S}| = s$. If, for $\epsilon \in (0,1)$,
$P_{\bar{\theta}} = \mathcal{U}((1-\epsilon)\bar{\theta} + \epsilon\Theta_{\bar{S}})$ and $Q_1^+ = \sum_{S \subset [d]: |S|=s}\Pi(S)Q_S$,
then $\DKL{P_{\bar{\theta}}}{Q_1^+} \leq s\log\frac{2ed}{\epsilon s}$.\label{lem:spike_kl}
\end{lemma}

\begin{proof}
We notice that $(1-\epsilon)\bar{\theta} + \epsilon\Theta_{\bar{S}}$ is an $s$-dimensional $\ell_1$-ball of radius $\epsilon$, which is contained in $\Theta_{\bar{S}}$. Therefore, on the support of $P_{\bar{\theta}}$, $\frac{\mathrm{d}P_{\bar{\theta}}}{\mathrm{d}Q_{\bar{S}}}$ is equal to the ratio of the volumes of a unit $\ell_1$-ball and an $\ell_1$-ball of radius $\epsilon$, which is $(1/\epsilon)^s$. Thus,
\begin{equation*}
\DKL{P_{\bar{\theta}}}{Q_1^+} = \int\log\frac{\mathrm{d}P_{\bar{\theta}}}{\sum_{S}\Pi(S)\mathrm{d}Q_{S}}\mathrm{d}P_{\bar{\theta}} \leq \int\log\frac{\mathrm{d}P_{\bar{\theta}}}{\Pi(\bar{S})\mathrm{d}Q_{\bar{S}}}\mathrm{d}P_{\bar{\theta}} \leq s\log\frac{1}{\epsilon} + \log\frac{1}{\Pi(\bar{S})}\,.
\end{equation*}

Using the definition of $\Pi$ and the bound $\binom{d}{s} \leq (\frac{ed}{s})^s$ on the binomial coefficient, we have
\begin{equation*}
\log\frac{1}{\Pi(\bar{S})} = \log\binom{d}{s} + s\log(2) + \log\sum_{k=1}^{s}2^{-k} \leq s\log\frac{2ed}{s}\,.
\end{equation*}

Combining everything, we obtain
\begin{equation}
\DKL{P_{\bar{\theta}}}{Q_1^+} \leq s\log\frac{1}{\epsilon} + s\log\frac{2ed}{s} = s\log\frac{2ed}{\epsilon s}\, ,
\end{equation}
as advertised.
\end{proof}

\section{Proof of the history-dependent part of Theorem~\ref{thm:Optimistic_posterior_RB}}
\label{app:history_dep_Optimistic_posterior}
Following the original analysis, we
arrive again at \eqref{eq:FTRL_decomposition}.
\begin{equation*}
	 \sum_{t=1}^{T}\Delta(A_t, P) \leq \frac{\Phi(P)}{\lambda_T} + \frac{\Phi^{*}(- \eta\Llik_T(\cdot ) + \eta\Llik_T(\theta_T) -
	\lambda_T \Lest_T(\cdot ))}{\lambda_T} + \frac{\eta}{\lambda_T}(\siprod{P}{\Llik_T} - \Llik_T(\theta_T)),
\end{equation*}
where $ P\in \Delta(\Theta) $ can be any comparator distribution.
Lemma~\ref{lemma:prior_comparator_bound} is still valid and we can chose the same prior as before. We can still choose a
comparator distribution supported on an $ \epsilon $-ball around $ \theta_0 $. However, because $ \lambda_t $ depends on the history, we can no longer upper bound $ \EE{\frac{|P\cdot
\Llik_T - \Llik_T(\theta_0)|}{\lambda_{T-1}}} $ by $ \EE{\frac{2T \epsilon}{\lambda_{T}}} $.
Using Lemma~\ref{lemma:lipschitzness}, we still have that $ \Lest_T(\cdot ) $ is $ 2T $-Lipschitz and $
\EE{\Llik_T(\cdot)} $ is $ 2T $-Lipschitz. Hence, 
\begin{equation*}
	\EE{\frac{|P\cdot \Llik_T - \Llik_T(\theta_0)|}{\lambda_{T-1}}} \leq  2T \epsilon C_{2,T}, \quad \text{and} \quad \sum_{t=1}^T|\Delta(\theta_0,
	a_t) - \Delta(P, a_t) | \leq 2T \epsilon,
\end{equation*}
where we used $ C_{2,T} $, a deterministic upper bound on $ \frac{1}{\lambda_{T-1}} $.
Exactly the same telescoping of $ \Phi^{*} $ can be done, however because the learning rate is history-dependent, the difference between
the negative log likelihood of $ \theta_0 $ and $ \theta_t $ must be treated with more care. We have the following lemma

\begin{lemma}
\label{lemma:MLE_bound_lambda}
Let $ C_{1,T} $ be a deterministic upper bound on $\left( \frac{1}{\lambda_{t+1}} - \frac{1}{\lambda_t} \right)$ that holds for all $
t< T $, then
\begin{align}
\label{eq:MLE_bound_lambda}
 \nonumber&\phantom{=}\EE{\sum_{t=1}^{T}\frac{\eta}{\lambda_{t-1}}(\Llik_t(\theta_t) - \Llik_t(\theta_0) + \Llik_{t-1}(\theta_0) -
	\Llik_{t-1}(\theta_{t-1}))+ \frac{\eta}{\lambda_T}(\Llik_T(\theta_0) - \Llik_T(\theta_T))}\\
	  &\leq \EE{\frac{\eta (15 + 3s \log \frac{2e^2 d T^2 C_{1,T} ^2}{s})}{2 \lambda_{T-1}}}.
\end{align}
\end{lemma}
A complete proof of that result can be found in appendix~\ref{app:MLE_bound_lambda}.

Finally, as was the case in the history independent version the telescoping sum can be handled by looking at the
explicit formula for $ \Phi^{*} $ and Lemma~\ref{lemma:telescoping_phi_s} still holds. Applying
Lemma~\ref{lemma:telescoping_phi_s} and setting $
\epsilon=\frac{1}{TC_{2,T}} $ yields the claim
of the theorem.

\section{Proof of Theorem~\ref{thm:instance_dependent_regret}}
\label{app:instance_dependent_regret}
We turn our attention to data-dependent bounds (that will scale with the cumulative information ratio rather than the
time horizon). Combining the second part of Theorem~\ref{thm:Optimistic_posterior_RB} with
Lemma~\ref{lemma:regret_GIR_bound} and the choice $ \lambda = \frac{64}{3} \lambda_{t-1} $, we have that for any non-increasing sequence of learning rates $ \lambda_t $ satisfying
$ \lambda_0 \leq \frac{1}{2} $, the following holds
\begin{equation}
\label{eq:RB_data_dep}
	R_T \leq  \EE{\frac{C_T}{\lambda_{T-1}}+ \min \left( \sum_{t=1}^{T}\frac{32}{3}\lambda_{t-1} \SIRs_t(\pi_t), \frac{16}{3}
	c_3^{*} \sqrt{3 \lambda_{t-1} \SIRc_t(\pi_t)} \right)},
\end{equation}
where $ C_T = 2 + s \log \frac{4e^3d^2 T^3 C_{1,T}^2 C_{2,T}}{s^2} $ and $ C_{1,T}$, respectively $ C_{2,T} $ are deterministic upper
bounds on $ \frac{1}{\lambda_t} - \frac{1}{\lambda_{t-1}} $, respectively $ \frac{1}{\lambda_{T-1}} $.

We let $ \lambdas_t = \sqrt{\frac{s}{2d+ \sum_{s=1}^{t} \SIRs_s(\pi_s)}} $ and $ \lambdac_t = \left(
\frac{s}{ \frac{3 \sqrt{6} s}{ \sqrt{C_{\min}}} + \sum_{s=1}^{t} \sqrt{\SIRc_s(\pi_s)}} \right)^\frac{2}{3} $, and verify that $ \lambda_t =
\max(\lambdas_t, \lambdac_t) $
is decreasing and always smaller than $ \frac{1}{2} $ . We also verify that $ C_{1,T} =  C_{2,T} = \sqrt{\frac{dT}{s}} $
are valid upper bounds. As a result, we have the following upper bound
\begin{equation}
\label{eq:bound_CT}
	C_T = 2 + s \log \frac{4e^3d^2 T^3 C_{1,T}^2 C_{2,T}}{s^2} \leq 2 + s \log 4 e^3 T^{4.5} \left( \frac{d}{s}
	\right)^{3.5} \leq 2 + 5s \log(\frac{edT}{s}).
\end{equation}
We now focus on bounding the sum containing the information ratios.
Applying Lemma~\ref{lemma:IR_bound}, we obtain that for all $ t\geq 1 $, $ \SIRs_t(\pi_t) \leq 2d $ and for any $ T\geq 1 $
\begin{align*}
	\sum_{t=1}^{T} \lambdas_{t-1} \SIRs_t(\pi) &= \sqrt{s} \sum_{t=1}^{T} \frac{\SIRs_t(\pi_t)}{\sqrt{2d +
	\sum_{s=1}^{t-1} \SIRs_s(\pi_s)}}\\
	&\leq  \sqrt{s} \sum_{t=1}^{T} \frac{\SIRs_t(\pi_t)}{
	\sqrt{\sum_{s=1}^{t} \SIRs_s(\pi_s)}}\\
&\leq 2 \sqrt{s \sum_{t=1}^{T} \SIRs_t(\pi_t)}\\
&\leq 2\sqrt{s \left( 2d + \sum_{t=1}^{T-1} \SIRs_t(\pi_t) \right)},
\end{align*}
where we applied Lemma~\ref{lemma:implicit_RB} with the function $ f(x)= \frac{1}{\sqrt{x}} $ and $ a_i = \SIRs_{i}(\pi_i) $ to
get the second inequality. This can be seen as a generalization of the usual $ \sum_{t=1}^{T} \frac{1}{\sqrt{t}} \leq
2\sqrt{T} $ inequality. We now define $ \Rs_T = \sqrt{s \left( 2d + \sum_{t=1}^{T-1} \SIRs_t(\pi_t) \right)} $, the
data-dependent regret rate associated to the 2-surrogate-information ratio.

We now turn our attention to the 3-information ratio. Applying Lemma~\ref{lemma:IR_bound} we obtain that for all $ t\geq
1$, $ \SIRc_t(\pi_t) \leq 54\frac{s}{C_{\min}} \leq 54\frac{s^2}{C_{\min}} $ and for any $ T\geq 1 $
\begin{align*}
	\sum_{t=1}^{T} \sqrt{\lambdac_{t-1} \SIRc_t(\pi_t)} &=   s^\frac{1}{3} \sum_{t=1}^{T}
	\frac{\sqrt{\SIRc_t(\pi_t)}}{\left( \frac{3 \sqrt{6}s}{\sqrt{C_{min}}} + \sum_{s=1}^{t-1} \sqrt{\SIRc_s(\pi_s)}
	\right)^\frac{1}{3}}\\
&\leq  s^\frac{1}{3} \sum_{t=1}^{T}
	\frac{\sqrt{\SIRc_t(\pi_t)}}{\left(  \sum_{s=1}^{t} \sqrt{\SIRc_s(\pi_s)}
	\right)^\frac{1}{3}}\\
&\leq  \frac{3}{2}s^\frac{1}{3}  \left(  \sum_{t=1}^{T} \sqrt{\SIRc_t{(\pi_t)}}\right)^\frac{2}{3}\\
& \leq   \frac{3}{2} s^ \frac{1}{3}  \left( \frac{3 \sqrt{6}s}{\sqrt{C_{min}}} + \sum_{t=1}^{T-1} \sqrt{\SIRc_t(\pi_t)}
\right),
\end{align*}
where we applied Lemma~\ref{lemma:implicit_RB} with the function $ f(x) = \frac{1}{x^{\frac{1}{3}}} $ and $ a_i =
\sqrt{\SIRc_{i}(\pi_i)} $ to get the second inequality. This can be seen as a generalization of the usual $
\sum_{t=1}^{T} \frac{1}{t^\frac{1}{3}} \leq \frac{3}{2} T^{\frac{2}{3}} $. We now define $ \Rc_T  = s^\frac{1}{3}
\left( 
\frac{3 \sqrt{6}s}{\sqrt{C_{min}}} + \sum_{t=1}^{T-1} \sqrt{\SIRc_t(\pi_t)} \right)^\frac{2}{3}$, the data-dependent regret rate
associated to the 3-surrogate-information ratio. We now consider the last time that the learning rates $ \lambdac_t $
and $ \lambdas_t $ have been used. More specifically, we denote $ T_2 = \max \{ t \leq T, \lambdas_{t-1} \geq \lambdac_{t-1}  \}  $, and $ T_3 = \max \{ t \leq T, \lambdac_{t-1} \geq
\lambdas_{t-1} \}
$.
Coming back to the bound of Equation~\ref{eq:RB_data_dep} and using the definition $ \lambda_t =\max(\lambdas_t,
\lambdac_t) ) $, the following bound holds
\begin{align*}
	&\phantom{=}R_T\\
	&\leq \EE{\frac{C_T}{\lambda_{T-1}} + \sum_{t=1}^{T} \min\left(\frac{32}{3} \lambda_{t-1}\SIRs_t(\pi_t), 
					  \frac{16}{3}c_3^{*} \sqrt{3\lambda_{t-1}\SIRc_t(\pi_t)}\right)}\\
&\leq\EE{ C_T \min \left( \frac{1}{\lambdas_{T-1}}, \frac{1}{\lambdac_{T-1}} \right) + \sum_{t=1}^{T}
	\min\left(\frac{32}{3}
	\max(\lambdas_{t-1}, \lambdac_{t-1}) \SIRs_t(\pi_t), \frac{16}{3}
					  c_3^{*} \sqrt{3\max(\lambdas_{t-1}, \lambdac_{t-1} )\SIRc_t(\pi_t)}\right)}.
\end{align*}

We can now separate the sum obtained at the last line based on which learning rate was used at time t.
\begin{align*}
	&\phantom{=}\sum_{t=1}^{T}
	\min\left(\frac{32}{3}
	\max(\lambdas_{t-1}, \lambdac_{t-1}) \SIRs_t(\pi_t), \frac{16}{3}
					  c_3^{*} \sqrt{3\max(\lambdas_{t-1}, \lambdac_{t-1} )\SIRc_t(\pi_t)}\right) \\
	&\leq \sum_{\lambdas_{t-1} \geq \lambdac_{t-1}} \frac{32}{3} \lambdas_{t-1} \SIRs_t(\pi_t) +
	\sum_{\lambdac_{t-1} \geq \lambdas_{t-1}}
	\frac{16}{3} c_3^{*} \sqrt{3 \lambdac_{t-1} \SIRc_t(\pi_t)}\\ 
	&\leq \sum_{t=1}^{T_2} \frac{32}{3} \lambdas_{t-1} \SIRs_t(\pi_t) + \sum_{t=1}^{T_3}
	\frac{16}{3} c_3^{*} \sqrt{3 \lambdac_{t-1} \SIRc_t(\pi_t)}.
\end{align*}

We further
bound $ \sum_{t=1}^{T_2} \frac{32}{3} \lambdas_{t-1} \SIRs_{t}(\pi_t) \leq  \frac{64}{3} \Rs_{T_2} $  and $
\sum_{t=1}^{T_3} \frac{16}{3} c_3^{*} \sqrt{3\lambdac_{t-1} \SIRc_t(\pi_t)} \leq \frac{16 }{3} \Rc_{T_3} $
(Using the explicit value $ c_3^{*} = \frac{2}{3^\frac{3}{2}} $).

The crucial observation is that which of $ \lambdac_T $ or $ \lambdas_T $ is bigger will determine whether $ \Rs_T $ or
$ \Rc_T $ is the term of leading order (up to some constants). More specifically,
Let $ T $ be such that $ \lambdas_{T-1} \geq \lambdac_{T-1} $ which means that $\sqrt{\frac{s}{2d +
 \sum_{t=1}^{T-1} \SIRs_t(\pi_t)}} \geq  \left(\frac{s}{\frac{3 \sqrt{6}s}{\sqrt{C_{\min}}} + \sum_{t=1}^{T-1}
\sqrt{\SIRc_t(\pi_t)}}\right)^\frac{2}{3}$. Rearranging,
this implies that $ \sqrt{s\big( 2d + \sum_{s=1}^{T-1} \SIRs_t(\pi_t) \big)}  \leq s^\frac{2}{3} \left(
	\frac{3 \sqrt{6}s}{\sqrt{C_{\min}}} + \sum_{t=1}^{T-1}\sqrt{\SIRc_t(\pi_t)}\right)^ \frac{2}{3} $,
which means that $ \Rs_{T} \leq \Rc_{T} $.
Following the exact same steps, we also have that $ \lambdac_{T-1} \geq \lambdas_{T-1} $ implies that $ \Rc_{T} \leq
 \Rs_{T} $.
We apply this to the time $ T_2 $ in which $ \lambdas_{T_2-1} \geq \lambdac_{T_2 -1} $ by definition. we have that $
\Rs_{T_2} \leq \Rc_{T_2}$ and putting this together with the previous bound, we have
\begin{align*}
	R_T &\leq \EE{\frac{C_T}{\lambdac_{T-1}} + \frac{64}{3} \Rs_{T_2} + \frac{16}{3} \Rc_{T_3}} \\
	    &\leq \EE{\frac{C_T}{s}\Rc_T + \frac{64}{3} \Rs_{T_2} + \frac{16}{3} \Rc_{T_3}} \\
	    &\leq \EE{\frac{C_T}{s} \Rc_T + \frac{64}{3} \Rc_{T_2} + \frac{16}{3} \Rc_{T_3}} \\
	    &\leq \EE{\frac{C_T}{s} \Rc_T + \frac{64}{3} \Rc_{T} + \frac{16}{3} \Rc_{T}} \\
	    &\leq \EE{\left( \frac{C_T}{s}  + \frac{80}{3} \right) \Rc_T},
\end{align*}
where we use the fact that $ T \mapsto \Rs_T $ and $ T \mapsto \Rc_T $ are non-decreasing and $ T_2 \leq T, T_3 \leq T $

Similarly by definition of $ T_3 $, we have that $ \lambdac_{T_3 -1} \geq \lambdas_{T_3 -1} $  and we can conclude that $
\Rc_{T_3} \leq  \Rs_{T_3} $.

Putting this together, with the previous bound, we have
\begin{align*}
	R_T &\leq \EE{\frac{C_T}{\lambdas_{T-1}} + \frac{64}{3} \Rs_{T_2} + \frac{16}{3} \Rc_{T_3}} \\
	    &\leq \EE{\frac{C_T}{s} \Rs_T + \frac{64}{3}\Rs_{T_2} + \frac{16}{3} \Rc_{T_3}} \\
	    &\leq \EE{\frac{C_T}{s} \Rs_T + \frac{64}{3} \Rs_{T_2} + \frac{16}{3} \Rs_{T_3}} \\
	    &\leq \EE{\frac{C_T}{s} \Rs_T + \frac{64}{3} \Rs_{T} + \frac{16}{3} \Rs_{T}} \\
	    &\leq \EE{( \frac{C_T}{s} + \frac{80}{3}) \Rs_T},
\end{align*}
where we use the fact that $ T \rightarrow \Rs_T $ and $ T\rightarrow \Rc_T $ are non-decreasing and $ T_2 \leq T, T_3
\leq T $.
Putting both of those bounds together with Equation~\ref{eq:bound_CT} yields the claim of the Theorem.

\section{Maximum likelihood estimation}
The focus of this section is to bound the difference between the log-likelihoods associated with the true parameter and 
the maximum likelihood estimator (MLE). We start by establishing an upper bound that holds in expectation
which suffices to handle history-independent learning rates. Then, we move on to high-probability bounds that will 
allow us to deal with data-dependent learning rates.
\subsection{Bound in expectation}
We start with the case in
which the maximum likelihood estimator is computed on a finite subset of the parameter space $ \Theta $.
\begin{lemma}
\label{lemma:MLE_bound_finite}
	Let $ t\geq 1 $, and $ \Theta' $ be a finite subset of $ \Theta $, we define the MLE over $ \Theta' $ as
	\begin{equation*}
		\theta_{\MLE, t}(\Theta') =
		\argmin_{\theta\in \Theta'} \Llik_t(\theta).
	\end{equation*}
	Then,
\begin{equation}
\label{eq:MLE_bound_finite}
	\EE{\Llik_t(\theta_0) - \Llik_t(\theta_{\MLE,t}(\Theta'))} \leq \log |\Theta'|
\end{equation}
\end{lemma}
\begin{proof}
 By the concavity of
the logarithm and Jensen's inequality, we have
\begin{align*}
	&\EE{\Llik_t(\theta_0) - \Llik_t(\theta_{\MLE,t}(\Theta'))} \leq \log 
\EE{\prod_{s=1}^{t}\frac{p(Y_s|\theta_{\MLE,t}(\Theta'),A_s)}{p(Y_s|\theta_0,A_s)}}\\
					 &\qquad\qquad\qquad= \log
					 \EE{\max_{\theta \in \Theta'}\prod_{s=1}^{t}\frac{p(Y_s|\theta,A_s)}{p(Y_s|\theta_0,A_s)}}\leq \log 
\EE{\sum_{\theta\in
					 \Theta'}\prod_{s=1}^{t}\frac{p(Y_s|\theta,A_s)}{p(Y_s|\theta_0,A_s)}}\\
					 &\qquad\qquad\qquad= \log\sum_{\theta\in \Theta'}
	\EE{\prod_{s=1}^{t}\frac{p(Y_s|\theta,A_s)}{p(Y_s|\theta_0,A_s)}}\\
\end{align*}
By Lemma~\ref{lemma:sub_gaussian_martingale}, we have that $ \exp \left(
\Llik_t(\theta_0) - \Llik_t(\theta)\right) = \prod_{s=1}^t \frac{p(Y_s|\theta, A_s)}{p(Y_s|\theta_0, A_s)}$ is a
non-negative supermartingale with respect to the filtration $ \F_t' = \sigma(\F_{t-1}, A_t) $. That implies that each
term in the sum is upper bounded by 1. Hence,
\begin{align*}
\EE{\Llik_t(\theta_0) - \Llik_t(\theta_{\MLE,t}(\Theta'))}
					 \leq  \log \sum_{\theta\in \Theta'} 1
					 = \log |\Theta'|,
\end{align*}
which proves the claim.
\end{proof}
To extend the previous bound to the full parameter space, we use a covering argument.
A subset $ \Theta' \subset \Theta $ is said to be a valid $ \rho $-covering of $ \Theta $ with respect to the $ \ell_1 $
norm if for every $
\theta \in \Theta $, there exists a $ \theta'\in \Theta' $ such that $ \norm{\theta - \theta'}_1 \leq  \rho $. We
denote by $ \mathcal{N}(\Theta, \|\cdot\|_1, \rho) $ the smallest possible cardinality of a valid $ \rho $ covering. We
have the following bound on this quantity.
\begin{lemma}
For every $\rho > 0$,
\begin{equation*}
\log \mathcal{N}(\Theta, \|\cdot\|_1, \rho) \leq \log\binom{d}{s}(1 + \tfrac{2}{\rho})^s \leq s \log \frac{ed(1 +
2/\rho)}{s}\,.
\end{equation*}
\label{lemma:covering_num}
\end{lemma}

\begin{proof}
For each subset $S \subset [d]$ of cardinality $|S| = s$, there is a surjective isometric embedding from $(\Theta_{S},
\|\cdot\|_1)$ to $(\mathbb{B}_1^s(1), \|\cdot\|_1)$.  In particular, to embed $\theta \in \Theta_{S}$ into
$\mathbb{B}_1^s(1)$, one can simply remove all the components of $\theta$ corresponding to indices not in $S$.
Therefore, for every $\rho > 0$, $\mathcal{N}(\Theta_{S}, \|\cdot\|_1, \rho) \leq \mathcal{N}(\mathbb{B}_1^s(1),
\|\cdot\|_1, \rho)$. Moreover, via a standard argument, we have $\mathcal{N}(\mathbb{B}_1^s(1),
\|\cdot\|_1, \rho) \leq (1 + \frac{2}{\rho})^s$ (see, e.g., Lemma 5.7 in \citealp{Wainw_19}). Now, let $\Theta_{S, 
\rho}$ denote any minimal $\rho$-covering of $\Theta_{S}$ and notice that for an arbitrary $\theta \in \Theta$ with 
support $S$, there exists a subset $\tilde{S}$ such that $S
\subseteq \tilde{S}$ and $|\tilde{S}| = s$. Therefore, there exists $\tilde{\theta} \in \Theta_{\tilde{S}, \rho}$ such
that $\|\theta - \tilde{\theta}\|_1 \leq \rho$. Hence, $\cup_{S \subset [d]: |S|=s}\Theta_{S, \rho}$ forms a valid 
$\rho$-covering of $\Theta$ and its cardinality is bounded by
\begin{align*}
\mathcal{N}(\Theta, \|\cdot\|_1, \rho) &\leq \left|\cup_{S \subset [d]: |S| = s}\Theta_{S, \rho}\right| \leq \sum_{S \subset [d]: |S| = s}\left(1 + \tfrac{2}{\rho}\right)^s = \binom{d}{s}\left(1 + \tfrac{2}{\rho}\right)^s\,.
\end{align*}
and we conclude by the elementary inequality $ \binom{d}{s} \leq \left( \frac{de}{s} \right)^s $.
\end{proof}

\subsubsection{Proof of Lemma~\ref{lemma:MLE_bound_complete}}
\label{app:MLE_bound_complete}
We bound the difference between the log-likelihood of the true parameter and that of the maximum likelihood estimator on
the full parameter space.
To this end, let $ \rho > 0$ and $ \Theta' $ be a minimal valid $ \rho $-cover of $ \Theta $ as is defined in
Lemma~\ref{lemma:covering_num}, and $ \theta'\in \Theta' $ be such that $ \norm{\theta' - \theta_t} \leq  \rho $, which
exists by definition of a $ \rho $-covering. Then,
\begin{align*}
	\EE{\Llik_t(\theta_0) - \Llik_t(\theta_t)} =& \EE{\Llik_t(\theta_0) - \Llik_t(\theta_{\MLE,t}(\Theta'))}\\
						   &+
	\EE{\Llik_t(\theta_{\MLE,t}(\Theta')) - \Llik_t(\theta')}\\
						   &+ \EE{\Llik_t(\theta') - \Llik_t(\theta_t)}\\
						   \leq& \log(\mathcal{N}(\Theta, \norm{\cdot }_1, \rho)) + 0 +
						   2\rho t,
\end{align*}
where the first term is bounded by Lemma~\ref{lemma:MLE_bound_finite}, the second term is non-positive by definition of the
maximum likelihood estimator because $ \theta' \in \Theta' $ and the third term is bounded because the mapping $ \theta 
\mapsto
\EE{\Llik_t(\theta)}$ is $ 2t $-Lipschitz with respect to the $1$-norm by Lemma~\ref{lemma:lipschitzness}. Finally 
applying
Lemma~\ref{lemma:covering_num} and setting $ \rho= \frac{2}{t} $ yields the desired bound. \qed

\subsection{High-probability bounds}
We begin with the case where the maximum likelihood estimator is computed over a finite subset of the parameter space $
\Theta$ and provide a corresponding high-probability bound.
\begin{lemma}
\label{lemma:MLE_bound_finite_hp}
Let $ \Theta' $ be a finite subset of $ \Theta $, we define $\theta_{MLE,t}( \Theta') =
\argmin_{\theta \in \Theta'} \Llik_t(\theta) $. Then
\begin{equation}
\label{eq:MLE_bound_finite_hp}
\PP{\exists t \geq  1, \Llik_t(\theta_0)  -\Llik_t(\theta_{MLE,t}(\Theta'))\geq \log \frac{|\Theta'|}{\delta}} \leq \delta	.
\end{equation}
\end{lemma}
\begin{proof}
Fix $ \theta \in \Theta' $. By Lemma~\ref{lemma:sub_gaussian_martingale}, we have that $ \exp \left(
\Llik_t(\theta_0) - \Llik_t(\theta)\right) = \prod_{s=1}^t \frac{p(Y_s|\theta, A_s)}{p(Y_s|\theta_0, A_s)}$ is a
non-negative supermartingale with respect to the filtration $ (\F_t^{\prime})_t $, where $ \F_t^{\prime} = \sigma(\F_{t}, A_{t+1}) $, allowing us to invoke 
Ville's inequality to get the following guarantee:
\begin{equation*}
	\PP{\exists t\geq 1, \exp(\Llik_t(\theta_0)- \Llik_t(\theta)) \geq \frac{1}{\delta}} \leq \delta.
\end{equation*}
Taking the logarithm and a union bound on $ \Theta' $ yields the desired result.
\end{proof}

We now provide a bound on the expected product of a bounded random variable with the differenece in log-likelihood
between the true parameter and the maximum likelihood estimator. 
\begin{lemma}
\label{lemma:MLE_bound_complete_dependent}
Let $ B \in \real $ and  $ X $ be a random variable satisfying $ 0 \leq X \leq B $ almost surely. Then for any $
t\geq 1 $,
\begin{align}
\label{eq:MLE_bound_complete_dependent}
\nonumber\EE{X(\Llik_t(\theta_0)- \Llik_t(\theta_t))} &\leq  \inf_{\delta, \rho > 0} \left\{ \EE{X s\log \frac{ed(1+ \frac{2}{\rho})}{s
	\delta^\frac{1}{s}}} + \frac{5}{2}B \rho t + B\delta s\log \frac{e^{1 + \frac{1}{s}}d (1 + \frac{2}{\rho})}{s
\delta^{\frac{1}{s}}} \right\}\\
					     &\leq 5 + s \log \frac{2e^2 d T^2 B^2}{s} \EE{X + \frac{1}{T}} .
\end{align}
\end{lemma}
\begin{proof}
Let $\delta, \rho >0 $ and $ \Theta' $ be a minimal valid $\rho$-cover of $ \Theta $ as defined in
Lemma~\ref{lemma:covering_num}, $ N= |\Theta'| $, let $ \theta' = \theta_{\MLE,t}(\Theta')
$ and let $ \bar{\theta}\in \Theta' $ be such that $ \norm{\bar{\theta}- \theta_t} \leq \rho $, which exists by
definition of a valid $ \rho $-cover. We have the following decomposition:
\begin{align*}
	\EE{X(\Llik_t(\theta_0)- \Llik_t(\theta_t))}
	\leq& \EE{X(\Llik_t(\theta_0)- \Llik_t(\theta')) \mathbf{1}_{\{ \Llik_t(\theta_0) - \Llik_t(\theta') \leq \log
	\frac{N}{\delta} \}}}\\
	&+ B\EE{(\Llik_t(\theta_0)- \Llik_t(\theta')) \mathbf{1}_{\{ \Llik_t(\theta_0) - \Llik_t(\theta') > \log
	\frac{N}{\delta} \}}}\\
	&+ B\EE{(\Llik_t(\bar{\theta})- \Llik_t(\theta_t))} + B\EE{( \Llik_t(\theta')-\Llik_t(\bar{\theta}))}.
\end{align*}
The first term is upper bounded by $ \EE{X \log \frac{N}{\delta}} $,
the third term is upper bounded by $ \frac{5}{2}B \rho t $ because $ \norm{\theta -\theta'}_1  $ is uniformely bounded
by $ \rho $ and by by
Lemma~\ref{lemma:lipschitzness}. The fourth term is non-positive because $ \theta' $ minimizes the negative log
likelihood on $ \Theta' $. Finally, we turn our attention to the second term. To simplify the computations, we let $ Y =
\Llik_t(\theta_0) - \Llik_t(\theta')$, and
 compute $ \EE{Y \mathbf{1}_{\{ Y > \log \frac{N}{\delta} \}}} $. Conditioning on whether $ \epsilon $ is larger or
 smaller than $ \log \frac{N}{\delta} $ yields the following identity
\begin{equation*}
\PP{Y\mathbf{1}_{\{ Y \geq \log \frac{N}{\delta} \}} \geq \epsilon} = 
\begin{cases}
\PP{Y \geq \epsilon} & \text{if } \epsilon \geq \log \frac{N}{\delta}, \\
\PP{Y \geq \log \frac{N}{\delta}} & \text{otherwise}.
\end{cases}
\end{equation*}
We can now upper bound the expectation as follows
\begin{align*}
	\EE{Y \mathbf{1}_{\{ Y\geq \log \frac{N}{\delta} \}}} &= \int_{0}^\infty \PP{Y \mathbf{1}_{\{ Y \geq \log
\frac{N}{\delta}} \}\geq \epsilon} \, d \epsilon\\
&= \log \frac{N}{\delta} \PP{Y\geq \log\frac{N}{\delta}} + \int_{\log \frac{N}{\delta}}^\infty \PP{Y \geq \epsilon} \, d
\epsilon\\
&= \log \frac{N}{\delta} \PP{Y\geq \log\frac{N}{\delta}} + \int_{0}^\delta \frac{1}{\delta'}\PP{Y \geq \log
\frac{N}{\delta'}} \, d \delta'\\
&\leq \delta \log \frac{N}{\delta} + \delta,
\end{align*}
where we used the change of variable $ \epsilon = \log \frac{N}{\delta'} $ and used $ \PP{Y\geq \log\frac{N}{\delta}}\leq \delta $ by Lemma~\ref{lemma:MLE_bound_finite_hp}.
Finally, putting everything together and using $ N \leq \mathcal{N}(\Theta, \norm{\cdot }_1, \rho) \leq \left(
\frac{ed(1 + \frac{2}{\rho})}{s} \right)^s$, by Lemma~\ref{lemma:covering_num}, we get

\begin{equation*}
	\EE{X(\Llik_t(\theta_0)- \Llik_t(\theta_t))} \leq \EE{X s\log \frac{ed(1+ \frac{2}{\rho})}{s
	\delta^\frac{1}{s}}} + \frac{5}{2}B \rho t + B\delta s\log \frac{e^{1 + \frac{1}{s}}d (1 + \frac{2}{\rho})}{s \delta^{\frac{1}{s}}}. 
\end{equation*}
To balance the trade-off between the approximation error and the covering complexity, we choose $ \rho =\frac{2}{BT} $,
and $ \delta = \frac{1}{BT} $ which yields the desired form of the logarithmic factors. Substituting these into 
the bound completes
the proof.
\end{proof}

\subsubsection{Proof of Lemma~\ref{lemma:MLE_bound_lambda}}
\label{app:MLE_bound_lambda}
As was noted in the analysis, since $ \lambda_T $ is not used by the
algorithm, we can replace $ \lambda_T$  by $ \lambda_{T-1} $ in our computations. We have
\begin{align*}
 &\EE{\sum_{t=1}^{T}\frac{\eta}{\lambda_{t-1}}(\Llik_t(\theta_t) - \Llik_t(\theta_0) + \Llik_{t-1}(\theta_0) -
	\Llik_{t-1}(\theta_{t-1}))+ \frac{\eta}{\lambda_T}(\Llik_T(\theta_0) - \Llik_T(\theta_T))}  \\	
 &\qquad\qquad= \EE{\sum_{t=1}^{T} \frac{\eta}{\lambda_{t-1}}(\Llik_t(\theta_t) - \Llik_t(\theta_0)) - \sum_{t=1}^{T}
 \frac{\eta}{\lambda_{t}}(\Llik_t(\theta_t) - \Llik_t(\theta_0))}\\
 &\qquad\qquad= \eta\cdot \sum_{t=1}^{T}\EE{(\Llik_t(\theta_0)- \Llik_t(\theta_t)) \left( \frac{1}{\lambda_{t}} -
		 \frac{1}{\lambda_{t-1}}
 \right)}.
 \end{align*}
Let $ C_{1,T} $ be a deterministic upper bound on $ \left( \frac{1}{\lambda_{t+1}}- \frac{1}{\lambda_t} \right) $.
Applying Lemma~\ref{lemma:MLE_bound_complete_dependent} to $ X = \left( \frac{1}{\lambda_{t+1}} - \frac{1}{\lambda_t} \right)
$ and telescoping, we get
 \begin{align*}
	 &\phantom{=} \eta\cdot \sum_{t=1}^{T}\EE{(\Llik_t(\theta_0)- \Llik_t(\theta_t)) \left( \frac{1}{\lambda_{t}} -
		 \frac{1}{\lambda_{t-1}}
 \right)}\\.
 &\leq \eta\left(5 + s \log \frac{2e^2 d t^2 C_{1,T} ^2}{s}\right) \sum_{t=1}^{T} \EE{ \left( \frac{1}{\lambda_{t}} -
 \frac{1}{\lambda_{t-1}} \right) + \frac{1}{T}}\\
 &\leq \eta\left(5 + s \log \frac{2e^2 d t^2 C_{1,T} ^2}{s}\right) \EE{\left(\frac{1}{\lambda_T} +1\right)}\\
 &\leq \EE{\frac{\eta (15 + 3s \log \frac{2e^2 d t^2 C_{1,T} ^2}{s})}{2 \lambda_{T-1}}},
\end{align*}
where in the last step, we used $ 1 \leq \frac{1}{2 \lambda_{T}} $ which implies $ \frac{1}{\lambda_T} + 1 \leq
\frac{3}{2\lambda_T} $. This finishes the proof. \qed

\section{Bounding the surrogate information ratio}
\subsection{Proof of Lemma~\ref{lemma:regret_GIR_bound}}
\label{app:regret_GIR_bound}
The surrogate regret of a policy is directly related to its $ 2 $- and $ 3 $-information ratio by definition
\begin{equation*}
	\Deltah_t(\pi) = \sqrt{\SIG_t(\pi) \SIRs_t(\pi)} = \left( \SIG_t(\pi) \SIRc_t(\pi) \right)^{\frac{1}{3}}.
\end{equation*}
By the AM-GM inequality, we have that for any $ \lambda>0 $, the surrogate regret is controlled as follows
\begin{equation*}
	\Deltah_t(\pi) \leq \frac{\SIG_t(\pi)}{\lambda} + \frac{\lambda}{4} \SIRs_t(\pi).
\end{equation*}
Similarly, by Lemma~\ref{lemma:Gen_AM_GM} which generalizes the AM-GM inequality, we can obtain the following regret
bound
\begin{equation*}
	\Deltah_t(\pi) \leq  \frac{\SIG_t(\pi)}{\lambda} + c_3^{*} \sqrt{\lambda \SIRc_t(\pi)},
\end{equation*}
where $ c_3^{*} <2 $ is an absolute constant defined in Lemma~\ref{lemma:Gen_AM_GM}. This concludes the proof. \qed
\subsection{Proof of Lemma~\ref{lemma:GIR_minimizer}}
\label{app:GIR_minimizer}

The proof of Lemma~\ref{lemma:GIR_minimizer} is essentially the same as the proof of Lemma 5.6 in \citet{Hao_L_D21a}, but we state it here for completeness. Throughout this proof, we use $\siprod{p}{f} = \sum_{a \in \A}p(a)f(a)$ to denote the inner product between a signed measure $p$ on $\A$ and a function $f: \A \rightarrow \real$. Using this notation, we can, for example, write the generalized surrogate information ratio as $\SIR_t^{(\gamma)}(\pi) = \siprod{\pi}{\SIR_t^{(\gamma)}}$.

We define $\pi_{t}^{(\gamma)} \in \argmin_{\pi \in \Delta(\A)}\SIR_t^{(\gamma)}(\pi)$ to be any minimizer of the 
generalized surrogate information ratio with parameter $\gamma \geq 2$. First, we observe that
\begin{equation*}
\nabla_{\pi}\SIR_t^{(2)}(\pi) = \frac{2\siprod{\pi}{\wh\Delta_t}\wh\Delta_t}{\siprod{\pi}{\SIG_t}} - \frac{(\siprod{\pi}{\wh\Delta_t})^2\SIG_t}{(\siprod{\pi}{\SIG_t})^2}\,.
\end{equation*}

Therefore, from the first-order optimality condition for convex constrained minimization (and the fact that $\SIR_t^{(2)}$ is convex on $\Delta(\A)$), we have
\begin{equation*}
\forall \pi \in \Delta(\A), ~0 \leq \siprod{\pi - \pIDS_t}{\nabla_{\pi}\SIR_t^{(2)}(\pIDS_t)}\,.
\end{equation*}

In particular,
\begin{equation*}
0 \leq \frac{2\siprod{\pIDS_t}{\wh\Delta_t}\siprod{\pi_t^{(\gamma)} - \pIDS}{\wh\Delta_t}}{\siprod{\pIDS_t}{\SIG_t}} - \frac{(\siprod{\pIDS_t}{\wh\Delta_t})^2\siprod{\pi_t^{(\gamma)} - \pIDS}{\SIG_t}}{(\siprod{\pIDS_t}{\SIG_t})^2}\,.
\end{equation*}

This inequality is equivalent to
\begin{equation*}
2\siprod{\pi_t^{(\gamma)}}{\wh\Delta_t} \geq \siprod{\pIDS_t}{\wh\Delta_t}\left(1 + \frac{\siprod{\pi_t^{(\gamma)}}{\SIG_t}}{\siprod{\pIDS_t}{\SIG_t}}\right) \geq \siprod{\pIDS_t}{\wh\Delta_t}\,.
\end{equation*}

From this inequality, we obtain
\begin{align*}
\frac{(\siprod{\pIDS_t}{\wh\Delta_t})^{\gamma}}{\siprod{\pIDS_t}{\SIG_t}} &= \frac{(\siprod{\pIDS_t}{\wh\Delta_t})^{2}(\siprod{\pIDS_t}{\wh\Delta_t})^{\gamma-2}}{\siprod{\pIDS_t}{\SIG_t}}\\
&\leq \frac{(\siprod{\pi_t^{(\gamma)}}{\wh\Delta_t})^{2}(\siprod{\pIDS_t}{\wh\Delta_t})^{\gamma-2}}{\siprod{\pi_t^{(\gamma)}}{\SIG_t}}\\
&\leq 2^{\gamma-2}\frac{(\siprod{\pi_t^{(\gamma)}}{\wh\Delta_t})^{\gamma}}{\siprod{\pi_t^{(\gamma)}}{\SIG_t}} = 
2^{\gamma - 2}\min_{\pi \in \Delta(\A)}\SIR_t^{(\gamma)}(\pi)\,,
\end{align*}
thus proving the claim. \qed

\subsection{Proof of Lemma~\ref{lemma:IR_bound}}
\label{app:IR_bound}
This section is focused on bounding the information ratios of the sparse optimistic information directed
sampling policy.
As is widely done in the information directed sampling literature, we will introduce a ``forerunner'' algorithm
with controlled surrogate information ratio. By Lemma~\ref{lemma:GIR_minimizer}, the SOIDS policy will then 
automatically inherit the bound of the forerunner.

As one of our forerunners, we will make use of the Feel-Good Thompson Sampling (FGTS) algorithm introduced by 
\citet{Zhang_22}. Letting $\widetilde{\theta}_t  \sim Q_t^{+} $, the FGTS policy is defined as 
\begin{equation}
\label{eq:FGTS_policy}
\pFGTS_t(a) = \PPt{a^{*}(\widetilde{\theta_t}) = a}.		
\end{equation}
Which can be seen as the policy obtained by sampling a parameter $ \widetilde{\theta_t} \sim Q_t^{+}$ and then picking the optimal action under
this parameter. Compared to the usual Thompson Sampling policy, this boils down to replacing the Bayesian posterior by the
optimistic posterior. Whenever the optimal action for $\theta$ is non-unique, we define $a^*(\theta)$ to be any optimal action with minimal 0-norm. If there are multiple optimal actions with minimal 0-norm, ties can be broken arbitrarily.


\subsubsection{Bounding the two information ratio}

We will now prove the first part of Lemma~\ref{lemma:IR_bound}, by showing that the information ratio of the
FGTS policy is bounded by the dimension. The proof is exactly the same as in the Bayesian setting as is done in
Proposition 5 of \citet{Russo_R16}, Lemma 7 of \citet{Neu_O_P_S22} or in Lemma 5.7 of \citet{Hao_L_D21a}, except the Bayesian posterior is replaced with the optimistic posterior. We provide the proof here for completeness.

Since we defined the surrogate information gain in terms of the model $\theta$, as opposed to the optimal action $a^*(\theta)$, we follow the proof of Lemma 7 in \citet{Neu_O_P_S22}. For brevity, we let $\alpha_a = \pFGTS_t(a) = \PPt{a^{*}(\widetilde{\theta_t}) = a}$. We define the $|\A| \times |\A|$ matrix $M$ by
\begin{equation*}
M_{a,a^{\prime}} = \sqrt{\alpha_a\alpha_{a^{\prime}}}(\mathbb{E}_t[r(a, \wt\theta_t)|a^*(\wt\theta_t) = a^{\prime}] - r(a, \bar{\theta}(Q_t^+)))\,.
\end{equation*}

Next, we relate the surrogate information gain and the surrogate regret to the Frobenius norm and the trace of $M$. First, we can lower bound the surrogate information gain of FGTS as
\begin{align*}
\SIG_t(\pFGTS_t) &= \frac{1}{2}\sum_{a \in \A}\alpha_a\int_{\Theta}(r(a, \bar{\theta}(Q_t^+)) - r(a, \theta))^2\mathrm{d}Q_t^+(\theta)\\
&= \frac{1}{2}\sum_{a \in \A}\alpha_a\int_{\Theta}\sum_{a^{\prime} \in \A}\mathbf{1}_{\{a^*(\theta) = a^{\prime}\}}(r(a, \bar{\theta}(Q_t^+)) - r(a, \theta))^2\mathrm{d}Q_t^+(\theta)\\
&= \frac{1}{2}\sum_{a \in \A}\sum_{a^{\prime} \in \A}\alpha_a\int_{\Theta}\mathbf{1}_{\{a^*(\theta) = a^{\prime}\}}\mathrm{d}Q_t^+(\theta)\mathbb{E}_t[(r(a, \bar{\theta}(Q_t^+)) - r(a, \wt\theta_t)|a^*(\wt\theta_t) = a^{\prime}]\\
&\geq \frac{1}{2}\sum_{a \in \A}\sum_{a^{\prime} \in \A}\alpha_a\alpha_{a^{\prime}}\left(r(a, \bar{\theta}(Q_t^+)) - \mathbb{E}_t[r(a, \wt\theta_t)|a^*(\wt\theta_t) = a^{\prime}]\right)^2\\
&= \frac{1}{2}\sum_{a \in \A}\sum_{a^{\prime} \in \A}M_{a,a^{\prime}}^2 = \frac{1}{2}\|M\|_F^2\,.
\end{align*}

Next, we can re-write the surrogate regret of FGTS as
\begin{align}
\wh\Delta_t(\pFGTS_t) &= \int_{\Theta}r(a^*(\theta), \theta)\mathrm{d}Q_t^+(\theta) - \sum_{a \in \A}\alpha_a\int_{\Theta}r(a, \theta)\mathrm{d}Q_{t}^+\label{eqn:fgts_sur_reg}\\
&= \int_{\Theta}\sum_{a \in \A}\mathbf{1}_{\{a^*(\theta) = a\}}r(a^*(\theta), \theta)\mathrm{d}Q_t^+(\theta) - \sum_{a \in \A}\alpha_ar(a, \bar{\theta}(Q_t^+))\nonumber\\
&= \sum_{a \in \A}\alpha_a\mathbb{E}_t[r(a, \wt\theta_t)|a^*(\wt\theta_t) = a] - \sum_{a \in \A}\alpha_ar(a, \bar{\theta}(Q_t^+))\nonumber\\
&= \mathrm{tr}(M)\nonumber\,.
\end{align}

Using Fact 10 from \citet{Russo_R16}, we bound $\SIR_t^{(2)}(\pFGTS_t)$ as
\begin{equation*}
\SIR_t^{(2)}(\pFGTS_t) = \frac{(\wh\Delta_t(\pFGTS_t))^2}{\SIG_t(\pFGTS_t)} \leq \frac{2(\mathrm{tr}(M))^2}{\|M\|_F^2} \leq 2\cdot\mathrm{rank}(M)\,.
\end{equation*}

All the remains is to show that $M$ has rank at most $d$. Enumerate the actions as $\A = \{a_1, \dots, a_{|\A|}\}$, and let $\mu_i = \mathbb{E}_t[\wt\theta_t|a^*(\wt\theta_t) = a_i]$. By linearity of expectation (and of the reward function), we can write
\begin{equation*}
M_{i, j} = \sqrt{\alpha_i\alpha_j}\siprod{\mu_i - \bar{\theta}(Q_t^+)}{a_j}\,.
\end{equation*}

Therefore, $M$ can be factorized as
\begin{equation*}
M = \begin{bmatrix}
\sqrt{\alpha_1}(\mu_1 - \bar{\theta}(Q_t^+))^{\top} \\ \vdots \\ \sqrt{\alpha_{|\A|}}(\mu_{|\A|} - \bar{\theta}(Q_t^+))^{\top}
\end{bmatrix} \begin{bmatrix}
\sqrt{\alpha_1}a_1 & \cdots & \sqrt{\alpha_{|\A|}}a_{|\A|}
\end{bmatrix}\,.
\end{equation*}

Since $M$ is the product of a $K \times d$ matrix and a $d \times K$ matrix, it must have rank at most $\min(K, d)$.

\subsubsection{Bounding the three information ratio}
\label{app:3_ir_bound}
To bound the 3 information ratio we follow \citet{Hao_L_D21a} and we introduce the exploratory policy
\begin{equation}
\label{eq:exploratory policy}
\mu = \argmax_{\pi \in \Delta(\A)} \sigma_{\min}\left(\sum_{a \in \A} \pi(a)a a^{\top}
\right).
\end{equation}
We define the mixture policy $ \pmix_t = (1- \gamma) \pFGTS_t + \gamma \mu $ where $ \gamma \geq  0$ will be
 determined later.
First, we lower bound the surrogate information gain of the mixture policy in the same way that we lower bounded the surrogate information gain of the FGTS policy previously. This time, we obtain the lower bound
\begin{align*}
\SIG_t(\pmix_t) &\geq \frac{1}{2}\sum_{a \in \A}\pmix_t(a)\sum_{a^{\prime} \in \A}\mathbb{P}_t(a^*(\wt\theta_t) = a^{\prime})(r(a, \bar{\theta}(Q_t^+)) - \mathbb{E}_t[r(a, \wt\theta_t)|a^*(\wt\theta_t) = a^{\prime}])^2\\
&= \frac{1}{2}\sum_{a \in \A}\pmix_t(a)\sum_{a^{\prime} \in \A}\mathbb{P}_t(a^*(\wt\theta_t) = a^{\prime})\siprod{\mu_{a^{\prime}} - \bar{\theta}(Q_t^+)}{a}^2\,,
\end{align*}

where $\mu_{a^{\prime}} = \mathbb{E}_t[\wt\theta_t|a^*(\wt\theta_t) = a^{\prime}]$. From the inequality $\pmix_t(a) \geq \gamma\mu(a)$, and the definition of $C_{\min}$, we have
\begin{align*}
\SIG_t(\pmix_t) &\geq \frac{\gamma}{2}\sum_{a^{\prime} \in \A}\mathbb{P}_t(a^*(\wt\theta_t) = a^{\prime})\sum_{a \in \A}\mu(a)(\mu_{a^{\prime}} - \bar{\theta}(Q_t^+))^{\top}a a^{\top}(\mu_{a^{\prime}} - \bar{\theta}(Q_t^+))\\
&\geq \frac{\gamma}{2}\sum_{a^{\prime} \in \A}\mathbb{P}_t(a^*(\wt\theta_t) = a^{\prime})C_{\min}\|\mu_{a^{\prime}} - \bar{\theta}(Q_t^+)\|_2^2\,.
\end{align*}

Using the expression for the surrogate regret of FGTS in \eqref{eqn:fgts_sur_reg}, we obtain
\begin{align*}
\wh\Delta_t(\pFGTS_t) &= \sum_{a \in \A}\mathbb{P}_t(a^*(\wt\theta_t)=a)(\mathbb{E}_t[\siprod{\wt\theta_t)}{a}|a^*(\wt\theta_t)=a] - \siprod{\bar{\theta}(Q_t^+)}{a})\\
&\leq \sqrt{\sum_{a \in \A}\mathbb{P}_t(a^*(\wt\theta_t)=a)(\mathbb{E}_t[\siprod{\wt\theta_t}{a}|a^*(\wt\theta_t)=a] - \siprod{\bar{\theta}(Q_t^+)}{a})^2}\,,
\end{align*}

where in the last we used the Cathy-Schwarz inequality. Due to the sparse optimal action property, all actions for which $\mathbb{P}_t(a^*(\wt\theta_t)=a) > 0$ have at most $s$ non-zero elements. Therefore,
\begin{equation*}
\sum_{a \in \A}\mathbb{P}_t(a^*(\wt\theta_t)=a)(\mathbb{E}_t[\siprod{\wt\theta_t}{a}|a^*(\wt\theta_t)=a] - \siprod{\bar{\theta}(Q_t^+)}{a})^2 \leq \sum_{a \in \A}\mathbb{P}_t(a^*(\wt\theta_t)=a)s\|\mu_a - \bar{\theta}(Q_t^+)\|_2^2\,.
\end{equation*}

This, combined with the lower bound on $\SIG_t(\pmix_t)$ means that
\begin{align*}
\wh\Delta_t(\pFGTS_t) &\leq \sqrt{\sum_{a \in \A}\mathbb{P}_t(a^*(\wt\theta_t)=a)s\|\mu_a - \bar{\theta}(Q_t^+)\|_2^2}\\
&= \sqrt{\frac{2s}{\gamma C_{\min}}\frac{\gamma}{2}\sum_{a \in \A}\mathbb{P}_t(a^*(\wt\theta_t)=a)C_{\min}\|\mu_a - \bar{\theta}(Q_t^+)\|_2^2}\\
&\leq \sqrt{\frac{2s}{\gamma C_{\min}}\SIG_t(\pmix_t)}\,.
\end{align*}

Choosing $\gamma = 1$, this tells us that
\begin{equation*}
(\wh\Delta_t(\pFGTS_t))^2 \leq \frac{2s}{C_{\min}}\SIG_t(\mu)\,.
\end{equation*}

We bound the information ratio in three cases. First, suppose that $\wh\Delta_t(\mu) \leq \wh\Delta_t(\pFGTS_t)$. In this case,
\begin{equation*}
\SIR_t^{(3)}(\mu) = \frac{\wh\Delta_t(\mu)(\wh\Delta_t(\mu))^2}{\SIG_t(\mu)} \leq \frac{2(\wh\Delta_t(\pFGTS_t))^2}{\SIG_t(\mu)} \leq \frac{4s}{C_{\min}}\,. 
\end{equation*}

Next, we consider the case where $\wh\Delta_t(\mu) > \wh\Delta_t(\pFGTS_t)$. For any $\gamma \in (0, 1]$,
\begin{equation*}
\SIR_t^{(3)}(\pmix_t) = \frac{((1-\gamma)\wh\Delta_t(\pFGTS_t) + \gamma\wh\Delta_t(\mu))^3}{(1-\gamma)\SIG_t(\pFGTS_t) + \gamma\SIG_t(\mu)} \leq \frac{((1-\gamma)\wh\Delta_t(\pFGTS_t) + \gamma\wh\Delta_t(\mu))^3}{\gamma\SIG_t(\mu)}\,.
\end{equation*}

We define $f(\gamma) = ((1-\gamma)\wh\Delta_t(\pFGTS_t) + \gamma\wh\Delta_t(\mu))^3/(\gamma\SIG_t(\mu))$ to be the RHS of the previous equation. One can verify that the derivative of $f(\gamma)$ is
\begin{equation*}
f^{\prime}(\gamma) = \frac{((1-\gamma)\wh\Delta_t(\pFGTS_t) + \gamma\wh\Delta_t(\mu))^2}{\gamma^2\SIG_t(\mu)}\left[2\gamma(\wh\Delta_t(\mu) - \wh\Delta_t(\pFGTS_t)) - \wh\Delta_t(\pFGTS_t)\right]\,,
\end{equation*}

and that $f(\gamma)$ is minimised w.r.t.\ $\gamma > 0$ at $\wh\gamma$, where $\wh\gamma$ is the positive solution of $f^{\prime}(\wh\gamma) = 0$, which is
\begin{equation*}
\wh\gamma = \frac{\wh\Delta_t(\pFGTS_t)}{2(\wh\Delta_t(\mu) - \wh\Delta_t(\pFGTS_t))}\,.
\end{equation*}

That $\wh\gamma$ is always positive follows from the fact that $\wh\Delta_t(\mu) > \wh\Delta_t(\pFGTS_t)$. If $\wh\gamma \leq 1$, then we can take the forerunner to be the mixture policy with $\gamma = \wh\gamma$. In this case,
\begin{align*}
\SIR_t^{(3)}(\pmix_t) &= \frac{(\frac{3}{2})^32(\wh\Delta_t(\mu) - \wh\Delta_t(\pFGTS_t))\wh\Delta_t(\pFGTS_t)^2}{\SIG_t(\mu)}\\
&\leq \frac{(\frac{3}{2})^38s}{C_{\min}} = \frac{27s}{C_{\min}}\,.
\end{align*}

Otherwise, if $\wh\gamma > 1$, then
\begin{equation*}
\wh\Delta_t(\mu) \leq \frac{3}{2}\wh\Delta_t(\pFGTS_t)\,.
\end{equation*}

In this case, we can take the forerunner to be $\mu$. The surrogate 3-information ratio can then be upper bounded as
\begin{equation*}
\SIR_t^{(3)}(\mu) = \frac{\wh\Delta_t(\mu)(\wh\Delta_t(\mu))^2}{\SIG_t(\mu)} \leq \frac{2(\frac{3}{2})^2(\wh\Delta_t(\pFGTS_t))^2}{\SIG_t(\mu)} \leq \frac{(\frac{3}{2})^24s}{C_{\min}} = \frac{9s}{C_{\min}}\,.
\end{equation*}

Therefore, one can always find a value of $\gamma \in (0,1]$ such that
\begin{equation*}
\SIR_t^{(3)}(\pmix_t) \leq \frac{27s}{C_{\min}}\,.
\end{equation*}

\section{Choosing the learning rates}
This section is focused on the choice of the learning rates required to obtain the bound of
Theorem~\ref{thm:regret_main}.
\subsection{Technical tools}
We start by a collection of technical results to help with choosing a time-dependent learning rate.
\begin{lemma}
\label{lemma:implicit_RB}
	Let $ a_i \geq 0 $ and $ f: [0, \infty) \rightarrow [0, \infty) $ be a nonincreasing function. Then
\begin{equation}
\label{eq:implicit_RB}
	\sum_{t=1}^{T} a_t f\left(\sum_{i=0}^{t} a_i\right) \leq \int_{a_0}^{\sum_{t=0}^{T}a_t} f(x) \, dx .
\end{equation}
\end{lemma}
The proof follows from elementary manipulations comparing sums and integrals. The result is taken from Lemma~4.13
of \citet{Orabo_19}, where a complete proof is also supplied. The following lemma ensures that the learning rates are 
non-increasing.
\begin{lemma}
\label{lemma:force_decreasing}
	Let $ C_1>e, C_2 >0 $ and define $ \lambda_t = \frac{\log(C_1 t)}{C_2 t} $, then $ \lambda_t $ is a
	non-decreasing
	sequence.
\end{lemma}
\begin{proof}
Let $ t>0 $, we have
\begin{equation*}
\frac{\log(C_1(t+1))}{\log (C_1 t)} = \frac{\log \left( C_1 t \left( \frac{t+1}{t} \right) \right)}{\log(C_1 t)} =
\frac{\log(C_1 t) + \log \left( \frac{t+1}{t} \right)}{ \log (C_1 t)} \leq 1 + \frac{1}{t \log(C_1 t)} \leq 1 +
\frac{1}{t},
\end{equation*}
where the first inequality uses $ \log(1+ x) \leq x $ for any $ x> -1  $ and the second inequality uses $ \log(C_1 t)
\geq  \log(C_1) \geq 1 $ because we assumed $ C_1 \geq e $. Since $ \frac{C_2 (t+1)}{C_2 t}= 1 + \frac{1}{t} $, we can
conclude that the sequence $ \lambda_t $ is non-increasing.
\end{proof}

\subsection{Data-rich regime: Proof of
Lemma~\ref{lemma:history_indep_learning_rates}}
\label{app:history_indep_learning_rates}
We start by focusing on the data rich regime, and we bound the following part of the regret 
bound given in Equation~\eqref{eq:RB_max}:
\begin{equation*}
	\frac{C_T}{\lambda_{T-1}} + \frac{32}{3}\sum_{t=1}^{T} \lambda_{t-1} \SIRs_t(\pi_t).
\end{equation*}
Here, $ C_T = 5 + 2s \log \frac{edT}{s} $. To proceed, we let $ \lambda_t = \alpha \sqrt{\frac{C_{t+1}}{d(t+1)}} $, 
where $ \alpha > 0$ is a constant that we will optimize later.
Because $ t \mapsto C_t $ is
increasing, we get that $ \lambda_{t-1} \leq \alpha
\sqrt{\frac{C_T}{dt}} $. By Lemma~\ref{lemma:IR_bound}, we know that for all $ t\geq 1 $, $ \SIRs_t(\pi_t) \leq 2d $,
hence
\begin{align*}
	\frac{C_T}{\lambda_{T-1}} + \frac{32}{3} \sum_{t=1}^{T} \lambda_{t-1} \SIRs_{t}(\pi_t) &\leq \frac{1}{\alpha}\sqrt{C_T
	dT} +
	\frac{64}{3} \alpha \sqrt{C_T} \sum_{t=1}^{T} \frac{d}{\sqrt{dt}}\\
& \leq \frac{1}{\alpha}\sqrt{C_T
	dT} +\frac{128}{3} \alpha \sqrt{C_TdT}\\
&\leq \left( \frac{1}{\alpha} + \frac{128}{3}\alpha \right) \sqrt{C_T d T}\\
&\leq 16 \sqrt{\frac{2}{3} C_T d T},
\end{align*}
where the second line uses the standard inequality $ \sum_{t=1}^{T} \frac{1}{\sqrt{t}} \leq  2 \sqrt{T} $, and the last 
line is obtained by optimizing the expression $ \left( \frac{1}{\alpha} + \frac{128}{3} \alpha \right) $ with the
optimal choice
$ \alpha = \sqrt{\frac{3}{128}}$ which yields the value $ 16 \sqrt{\frac{2}{3}} $. This concludes the proof of the claim.\qed

\subsection{Data-poor regime: proof of Lemma~\ref{lemma:history_indep_learning_rates}}
We now focus on the data-poor regime and specifically on bounding the following part of the bound given in 
Equation~\eqref{eq:RB_max}:
\begin{equation*}
	\frac{C_T}{\lambda_{T-1}} + \frac{16}{3} c_3^{*}\sum_{t=1}^{T} \sqrt{3\lambda_{t-1} \SIRc_t(\pi_t)}.
\end{equation*}
Here, $ C_T = 5 + 2s \log \frac{edT}{s} $. Now, we let $ \lambda_t = \alpha \left( \frac{C_{t+1}\sqrt{ 
C_{\min}}}{(t+1) \sqrt{s}}  \right)^{\frac{2}{3}} $, where $ \alpha > 0$ is a constant
that we will optimize later.
Because $ t \rightarrow C_t $ is
increasing, we get that $ \lambda_{t-1} \leq \alpha
\left(\frac{C_T \sqrt{C_{\min}}}{ts}\right)^\frac{2}{3} $.
By Lemma~\ref{lemma:IR_bound}, the 3-surrogate-information ratio is bounded for all $ t\geq 1 $ as $ \SIRc_t(\pi_t)
\leq \frac{54s}{C_{\min}} $. Hence,
the following holds:
\begin{align*}
	\frac{C_T}{\lambda_{T-1}} + \frac{16}{3} c_3^{*} \sum_{t=1}^{T} \sqrt{ 3\lambda_{t-1} \SIRc_{t}(\pi_t)} &\leq \frac{1}{\alpha}
	(C_T)^{\frac{1}{3}} \left( \frac{T \sqrt{s}}{\sqrt{C_{\min}}} \right)^{\frac{2}{3}} + 48c_3^{*}
	\sqrt{2\alpha}
	(C_T)^{\frac{1}{3}}\left( \frac{\sqrt{s}}{\sqrt{C_{\min}}} \right)^{\frac{2}{3}}  \sum_{t=1}^{T}
	\frac{1}{t^{\frac{1}{3}}}\\
&\leq \frac{1}{\alpha}
	(C_T)^{\frac{1}{3}} \left( \frac{T \sqrt{s}}{\sqrt{ C_{\min}}} \right)^{\frac{2}{3}} + 72 c_3^{*} \sqrt{2\alpha}
	(C_T)^{\frac{1}{3}}\left( \frac{T \sqrt{s}}{\sqrt{ C_{\min}}} \right)^{\frac{2}{3}}  \\
&\leq \left( \frac{1}{\alpha}+  72 c_3^{*} \sqrt{2\alpha}\right)
	(C_T)^{\frac{1}{3}} \left( \frac{T \sqrt{s}}{\sqrt{ C_{\min}}} \right)^{\frac{2}{3}}\\
&\leq  12 \cdot 6^{\frac{1}{3}}(C_T)^{\frac{1}{3}} \left( \frac{  T\sqrt{s}}{\sqrt{C_{\min}}} \right)^{\frac{2}{3}}.
\end{align*}
Here, we have applied Lemma~\ref{lemma:implicit_RB} with  
the function $ f(x) = x^{\frac{1}{3}} $and $ a_i = 1 $ to bound $ \sum_{t=1}^{T} {t^{-1/3}}  \leq \frac{3}{2} T^\frac{2}{3}$ in the 
second line, the last line comes
from the choice $\alpha = \frac{1}{4 \cdot 6^\frac{1}{3}}$ which optimizes the constant $ \left( \frac{1}{\alpha} + 72 c_3^{*}
 \sqrt{2\alpha}\right) $ (as per Lemma~\ref{lemma:Gen_AM_GM}). This proves the statement. \qed
 
\subsection{Joint learning rates, end of the proof of Theorem~\ref{thm:regret_main}}
\label{app:joint_learning_rate}
In the section below, we present the technical derivation related to choosing the choice of learning rate
$ \lambda_t=\min(\frac{1}{2},\max(\lambdas_t, \lambdac_t) ) $, where
$ \lambdas_t =  \sqrt{\frac{3C_{t+1}}{128d(t+1)}} $ and $ \lambdac_t = \frac{1}{4\cdot 6^\frac{1}{3}}\left( \frac{
	C_{t+1}\sqrt{C_{\min}}}{(t+1)\sqrt{s}}
\right)^{\frac{2}{3}} $, with $ C_t = 5 + 2s \log \frac{ ed t}{s} $. This choice interpolates between the data-rich and
data-poor regimes. As a first step, we start by confirming via Lemma~\ref{lemma:force_decreasing} that both $ 
\lambdas_t $ and $ \lambdac_t $ are non-increasing and the bound of Theorem~\ref{thm:Optimistic_posterior_RB} holds 
with our choice of $ \lambda_t $.

First, note that our choice of learning rates ensures that $\lambda_t \le \frac 
12$ holds as long as $T$ is larger than an absolute constant, and thus we focus on this case here (and relegate the 
complete details of establishing this absolute constant to Appendix~\ref{app:upper_bound_lr}). To proceed, 
we define the (constant-free) regret rates $ \Rs_t = \sqrt{C_tdt}
$ and  $ \Rc_t = \left( t \sqrt{s \frac{C_t}{C_{\min}}} \right)^\frac{2}{3}$ and note that they correspond to the regret 
bounds obtained when using the respective learning rates $ \lambdas_t $ and $ \lambdac_t $, as per
Lemma~\ref{lemma:history_indep_learning_rates}.

We now consider the last time that the learning rates $ \lambdac_t $ and $ \lambdas_t $  have been used. More
specifically, we denote 
$ T_2 = \max \{ t \leq T, \lambdas_{t-1} \geq \lambdac_{t-1}  \}  $, and $ T_3 = \max \{ t \leq T, \lambdac_{t-1} \geq
\lambdas_{t-1} \}
$.
Combining the bound of Equation~\ref{eq:RB_max} and using the definition $ \lambda_t =\min(\frac{1}{2},\max(\lambdas_t,
\lambdac_t) ) $, the following bound holds
\begin{align*}
	&\phantom{=}R_T\\
	&\leq \EE{\frac{C_T}{\lambda_{T-1}} + \sum_{t=1}^{T} \min\left(\frac{32}{3} \lambda_{t-1}\SIRs_t(\pi_t), 
					  \frac{16}{3}c_3^{*} \sqrt{3\lambda_{t-1}\SIRc_t(\pi_t)}\right)}\\
&=\mathbb{E}\left[\frac{C_T}{\min(\frac{1}{2},\max(\lambdas_{T-1}, \lambdac_{T-1}) )}\right.\\
&+ \left.\sum_{t=1}^{T} \min\left(\frac{32}{3}
	\min(\frac{1}{2},\max(\lambdas_{t-1}, \lambdac_{t-1}) )\SIRs_t(\pi_t), \frac{16}{3}
					  c_3^{*} \sqrt{ 3\min(\frac{1}{2},\max(\lambdas_{t-1}, \lambdac_{t-1})
			  )\SIRc_t(\pi_t)}\right)\right]\\
&\leq\EE{ C_T \min \left( \frac{1}{\lambdas_{T-1}}, \frac{1}{\lambdac_{T-1}} \right) + \sum_{t=1}^{T}
	\min\left(\frac{32}{3}
	\max(\lambdas_{t-1}, \lambdac_{t-1}) \SIRs_t(\pi_t), \frac{16}{3}
					  c_3^{*} \sqrt{3\max(\lambdas_{t-1}, \lambdac_{t-1} )\SIRc_t(\pi_t)}\right)}.
\end{align*}

We can now separate the sum obtained at the last line based on which learning rate was used at time t.
\begin{align*}
	&\phantom{=}\sum_{t=1}^{T}
	\min\left(\frac{32}{3}
	\max(\lambdas_{t-1}, \lambdac_{t-1}) \SIRs_t(\pi_t), \frac{16}{3}
					  c_3^{*} \sqrt{3\max(\lambdas_{t-1}, \lambdac_{t-1} )\SIRc_t(\pi_t)}\right) \\
	&\leq \sum_{\lambdas_{t-1} \geq \lambdac_{t-1}} \frac{32}{3} \lambdas_{t-1} \SIRs_t(\pi_t) + \sum_{\lambdac_t \geq \lambdas_t}
	\frac{16}{3} c_3^{*} \sqrt{3 \lambdac_{t-1} \SIRc_t(\pi_t)}\\ 
	&\leq \sum_{t=1}^{T_2} \frac{32}{3} \lambdas_{t-1} \SIRs_t(\pi_t) + \sum_{t=1}^{T_3}
	\frac{16}{3} c_3^{*} \sqrt{3 \lambdac_{t-1} \SIRc_t(\pi_t)}.
\end{align*}
Following exactly the same step as in the proof of Lemma~\ref{lemma:history_indep_learning_rates}, we further
bound $ \sum_{t=1}^{T_2} \frac{32}{3} \lambdas_{t-1} \SIRs_{t}(\pi_t) \leq  8\sqrt{\frac{2}{3}} \Rs_{T_2} $  and $
\sum_{t=1}^{T_3} \frac{16}{3} c_3^{*} \sqrt{3\lambdac_{t-1} \SIRc_t(\pi_t)} \leq  8 \cdot  6^\frac{1}{3}\Rc_{T_3} $.

The crucial observation is that which of $ \lambdac_T $ or $ \lambdas_T $ is bigger will determine whether $ \Rs_T $ or
$ \Rc_T $ is the term of leading order (up to some constants). More specifically,
Let $ T $ be such that $ \lambdas_{T-1} \geq \lambdac_{T-1} $ which means that $ \sqrt{\frac{3
C_{T}}{128 d T}} \geq \frac{1}{4 \cdot 6 ^\frac{1}{3}} \left( \frac{C_{T} \sqrt{C_{\min}}}{T \sqrt{s}} \right)^\frac{2}{3} $. Rearranging,
this implies that $ \sqrt{C_{T} d T} \leq \frac{6^\frac{5}{6}}{4} \left( T \sqrt{s \frac{C_{T}}{C_{\min}}} \right)^ \frac{2}{3} $,
which means that $ \Rs_{T} \leq \frac{6^\frac{5}{6}}{4} \Rc_{T} $.
Following the exact same steps, we also have that $ \lambdac_{T-1} \geq \lambdas_{T-1} $ implies that $ \Rc_{T} \leq
\frac{4}{6^\frac{5}{6}} \Rs_{T} $.
We apply this to the time $ T_2 $ in which $ \lambdas_{T_2-1} \geq \lambdac_{T_2 -1} $ by definition. we have that $
\Rs_{T_2} \leq \frac{6^\frac{5}{6}}{4}\Rc_{T_2}$ and putting this together with the previous bound, we have
\begin{align*}
	R_T &\leq \frac{C_T}{\lambdac_{T-1}} + 8\sqrt{\frac{2}{3}} \Rs_{T_2} + 8\cdot 6^\frac{1}{3} \Rc_{T_3} \\
	    &\leq 4 \cdot 6^\frac{1}{3} \Rc_T + 8\sqrt{\frac{2}{3}}\cdot \frac{6^\frac{5}{6}}{4}  \Rs_{T_2} + 8\cdot
	    6^\frac{1}{3}\Rc_{T_3} \\
	    &\leq 4\cdot 6^\frac{1}{3} \Rc_T + 4 \cdot 6^\frac{1}{3} \Rc_{T_2} + 8\cdot 6^\frac{1}{3} \Rc_{T_3} \\
	    &\leq 4\cdot 6^\frac{1}{3} \Rc_T +  4\cdot 6^\frac{1}{3}\Rc_{T} +  8\cdot 6^\frac{1}{3}\Rc_{T} \\
	    &\leq  16\cdot 6^\frac{1}{3}\Rc_T,
\end{align*}
where we use the fact that $ T \rightarrow \Rc_T $ is increasing and $ T_2 \leq T, T_3 \leq T $.

Using the same argument as before, we have that $ \lambdac_{T_3 -1} \geq \lambdas_{T_3 -1} $, and we can conclude that $
\Rc_{T_3} \leq  \frac{4}{6^\frac{5}{6}}\Rs_{T_3} $.

Putting this together, with the previous bound, we have
\begin{align*}
	R_T &\leq \frac{C_T}{\lambdas_{T-1}} + 8\sqrt{\frac{2}{3}} \Rs_{T_2} + 8\cdot 6^\frac{1}{3} \Rc_{T_3} \\
	    &\leq  8\sqrt{\frac{2}{3}}\Rs_T +  8\sqrt{\frac{2}{3}}\Rs_{T_2} + 8\cdot 6^\frac{1}{3}\cdot
	    \frac{4}{6^\frac{5}{6}}  \Rc_{T_3} \\
	    &\leq  8\sqrt{\frac{2}{3}}\Rs_T +  8\sqrt{\frac{2}{3}}\Rs_{T_2} + 16\sqrt{\frac{2}{3}} \Rs_{T_3} \\
	    &\leq  8\sqrt{\frac{2}{3}}\Rs_T +  8\sqrt{\frac{2}{3}}\Rs_{T} +  16 \sqrt{\frac{2}{3}}\Rs_{T} \\
	    &\leq   32 \sqrt{\frac{2}{3}}\Rs_T,
\end{align*}
where we use the fact that $ T \rightarrow \Rc_T $ is increasing and $ T_2 \leq T, T_3 \leq T $. Evaluating the
constants numerically yields $ 16 \cdot 6^\frac{1}{3} \approx 29.07 \leq 30 $ and  $ 32 \sqrt{\frac{2}{3}}\approx 26.13
\leq 27$.

\subsection{Upper bound on the learning rates}
\label{app:upper_bound_lr}

We now consider the case where the learning rates exceed $\frac 12$, and show that this only holds for small values of 
$T$. First, we have that $ \lambdas_{T-1} \leq \frac{1}{2}$ if 
\begin{equation*}
	\sqrt{\frac{3C_T}{128dT}} \leq \frac{1}{2}.
\end{equation*}
Rearranging the inequality and recalling $ C_T = 5 + 2s \log \frac{edT}{s} $, this is equivalent to
\begin{equation*}
	T \geq \frac{15}{32d} + \frac{3s}{16d} \log \frac{edT}{s}. 
\end{equation*}
Using the loose inequality $ \log \frac{edT}{s} \leq \frac{dT}{s} $, we get that this condition is satisfied for any $
T\geq 1 $.

Similarly, we have that $ \lambdac_{T-1} \leq \frac{1}{2} $ if
\begin{equation*}
	\frac{1}{4\cdot 6^\frac{1}{3}}\left( \frac{C_T \sqrt{C_{\min}}}{T \sqrt{s}} \right)^{\frac{2}{3}} \leq \frac{1}{2}.
\end{equation*}
We note that 
\begin{align*}
	C_{\min}&=\max_{\mu \in \Delta(A)} \sigma_{\min}(\EEs{A A^{\top}}{A \sim \mu}) 
		\leq \max_{\mu \in \Delta(A)} \frac{\text{Tr}(\EEs{A A^{\top}}{A \sim \mu})}{d} 
		\leq 1,
\end{align*}
where the first inequality uses that the trace of a matrix is always bigger than $ d $-times its smallest eigenvalue and
the second inequality uses the fact that for any vector $ a $, we have $\text{Tr}(aa^{\top}) = \sum_{i=1}^{d} a_i^2 \leq d 
\max_i
|a_i| \leq d $ because we assumed that all the actions are bounded in infinity norm.
Hence the previous inequality will be satisfied if 
\begin{equation*}
	\frac{1}{4\cdot 6^\frac{1}{3}}\left( \frac{C_T}{T \sqrt{s}} \right)^{\frac{2}{3}} \leq \frac{1}{2}.
\end{equation*}

Rearranging the inequality, this is equivalent to
\begin{equation*}
	T\geq 4 \sqrt{\frac{3}{s}} C_t  = 8\sqrt{3s}\log(eT) + \sqrt{3s}\left(\frac{20}{s} + 8 \log \frac{d}{s}\right).
\end{equation*}
Applying Lemma~\ref{lemma:W_log_bound} with $ a = 8\sqrt{3s} $ and $ b= \sqrt{3s} \left( \frac{20}{s} + 8 \log(\frac{d}{s}) \right)$, we find that the previous
inequality is satisfied for all
\begin{equation*}
	T\geq 2a \log ea + 2b = 40 \sqrt{\frac{3}{s}} + 16 \sqrt{3s}\log \frac{8e \sqrt{3} d}{\sqrt{s}}.
\end{equation*}

Thus, letting $T_{\min} = 40 \sqrt{\frac{3}{s}} + 16 \sqrt{3s}\log \frac{8e \sqrt{3} d}{\sqrt{s}}$ be the constant given above, both learning rates stay upper bounded by $ 
\frac{1}{2} $ for all $T \ge T_{\min}$ and the upper bound on the regret given the previous subsection holds. 
Otherwise, we upper bound the instantaneous regret
by $ 2 $ and this leads to an additional $ 2T_{\min} = \OO(\sqrt{s} \log \frac{d}{\sqrt{s}})$ in the regret.
Putting this together with the bound proved in the previous section, we thus have that the following regret 
bound is valid for any $ T\geq 1 $:
\begin{equation*}	
R_T \leq \min \left( 27 \sqrt{\left( 5 + 2s \log \frac{edT}{s} \right)d T}, 30 \left( 5 + 2s \log\frac{edT}{s}
\right)^{\frac{1}{3}} \left( \frac{T \sqrt{s}}{\sqrt{C_{\min}}} \right)^{\frac{2}{3}} \right) + \OO\left(\sqrt{s} \log
\frac{d}{\sqrt{s}}\right).
\end{equation*}
This concludes the proof of Theorem~\ref{thm:regret_main}. \qed

\setcounter{section}{8}
\section{Technical Results}
\label{app:technical_results}
We state and prove the remaining technical results.
\begin{lemma}
\label{lemma:lipschitzness}
Let  $ \pi \in \Delta(\A) $, the function $ \theta\mapsto \Delta(\pi, \theta) $ is 2-Lipschitz with respect to the
1 norm.
Let $ t\geq 1 $, the
function $ \theta \rightarrow \EE{\log\left(\frac{1}{p_t(Y_t|\theta,A_t)}\right)} $ is 2-Lipschitz with respect to the 1
norm. Moreover if $ \theta, \theta' $ are random variables, then 
\begin{equation*}
    \abs{\EE{\log \left( \frac{1}{p_t (Y_t|\theta, A_t)} \right) - \log \left( \frac{1}{p_t(Y_t|\theta',A_t)} \right)}} \leq
    \frac{5}{2} \sup \norm{\theta - \theta'}_1.
\end{equation*}

\end{lemma}
\begin{proof}
Let $ \theta, \theta' \in \Theta $, we have
\begin{align*}
	|r(\pi, \theta) - r(\pi, \theta')| &= \left\lvert \sum_{a \in \A} \pi(a) \siprod{\theta -
	\theta'}{a}\right\rvert\\
					    &\leq \sum_{a \in \A} \pi(a) |\siprod{\theta - \theta'}{a}|\\
					    &\leq \sum_{a \in \A} \pi(a) \norm{\theta - \theta'}_{1} \norm{a}_{\infty}\\
					    &\leq \norm{\theta - \theta'}_{1}.
\end{align*}
Similarly,
\begin{equation*}
	|r^{*}(\theta) - r^{*}(\theta')| = |\max_{a \in \A}r(a, \theta) - \max_{a \in \A}r(a, \theta')| \leq \max_{a\in \A}
	|r(\theta,a) - r(a, \theta')| \leq \norm{\theta - \theta'}_{1}.
\end{equation*}
Finally 
\begin{equation*}
	|\Delta(\pi, \theta) - \Delta(\pi, \theta')| = |r^{*}(\theta) - r^{*}(\theta') + r(\pi, \theta') - r(\pi,
	\theta)| \leq  2 \norm{\theta - \theta'}_{1}. 
\end{equation*}

Let us write $ r = \siprod{\theta }{A_t} $, $ r' = \siprod{\theta'}{A_t} $
and $ r_0 = \siprod{\theta_0}{A_t} $. We have that $ Y_t = r_0 + \epsilon_t $ where $ \epsilon_t $ is 1-sub-Gaussian.
For the negative log-likelihood, we then have
\begin{align*}
	\EE{\log \left( \frac{1}{p(Y_t |\theta,A_t)} \right) - \log \left( \frac{1}{p(Y_t| \theta', A_t)} \right)}
	&=\frac{1}{2}
	\EE{(\siprod{\theta}{A_t}- Y_t)^2 - (\siprod{\theta'}{A_t}- Y_t)^2}\\
										&=\frac{1}{2}\EE{(r-Y_t)^2 - (r' - Y_t)^2}\\
										&=\frac{1}{2}\EE{(r-r')(r + r' - 2
										Y_t)}\\
										&= \frac{1}{2} \EE{(r - r')(r + r' -
										2r_0)} + \frac{1}{2} \EE{(r-r')
                                                                                \epsilon_t}\\
                                                                                &\leq 2 \EE{\norm{ \theta - \theta'}_1}
                                                                                + \frac{1}{2}\EE{\norm{\theta-
                                                                                \theta'}_1 \abs{\epsilon_t}},
\end{align*}
where we use the fact that $ \abs{r + r' - 2 r_0}\leq 4 $.

Now if $ \theta, \theta' $ are fixed, since $ \EE{\epsilon_t} = 0 $, we get 
\begin{equation*}
    \EE{\log \left( \frac{1}{p(Y_t |\theta,A_t)} \right) - \log \left( \frac{1}{p(Y_t| \theta', A_t)} \right)} \leq 2
    \norm{\theta - \theta' }_1
\end{equation*}
If $ \theta, \theta' $ are random variables, the subgaussianity of $ \epsilon_t $ implies that $ \EE{\epsilon_t^2} \leq
1 $ and by Cauchy-Schwarz we have
\begin{equation*}
    \EE{\norm{\theta- \theta'}_1 \abs{\epsilon_t}} \leq \sqrt{\EE{ \norm{\theta- \theta'}_1^2} \EE{\epsilon_t^2}} \leq
    \sup{\norm{\theta - \theta'}_1}.
\end{equation*}
This yields the second claim of the Lemma.
\end{proof}

\begin{lemma}
(Hoeffding's Lemma)
\label{lemma:Hoeffdings_Inequality}
	Let $ X $ be a bounded real random variable such that $ X \in [a,b] $ almost surely. Let $\eta \neq 0  $, then we have
\begin{equation}
\label{eq:Hoeffidng}
	\frac{1}{\eta}\log \EE{\exp{(\eta X)}} \leq \EE{X} + \frac{\eta(b-a)^2}{8}.
\end{equation}

\end{lemma}
\begin{proof}
See for instance Chapter 2 in \citet{Bouch_L_M13}. 
\end{proof}
We now provide a data dependent version of Hoeffding's lemma that is used in the analysis of the gaps in the optimistic
posterior.

\begin{lemma}
(A data dependent version of Hoeffding's Lemma)
\label{lemma:data_dep_hoeffding}
Let $ X $ be a real random variable and $ \eta \neq 0 $ be such that $ \eta X \leq 1 $ almost surely, then we have
\begin{equation}
\label{eq:data_dep_hoeffding}
	\frac{1}{\eta} \log \EE{\exp{(\eta X)}} \leq \EE{X} + \eta \EE{X^2} \leq  2 \EE{X}.
\end{equation}
\end{lemma}
\begin{proof}
	Using the elementary inequalities $ \log(x) \leq x-1 $ for $ x>0 $ and $ e^x \leq 1 + x + x^2 $ for $ x\leq 1 $,
	we get that
\begin{align*}
	\frac{1}{\eta}\log \EE{\exp{(\eta X)}} &\leq \frac{1}{\eta}\EE{\exp(\eta X)-1} \\
					       &\leq \frac{1}{\eta} \EE{\eta X + \eta^2 X^2} \\
					       &\leq \EE{X} + \eta \EE{X^2}.
\end{align*}

\end{proof}

The following lemmas help us to analyze when the learning rates are smaller or bigger than $ \frac{1}{2} $.

\begin{lemma}
\label{lemma:W_log_bound}
Let $ a\geq 1, b\geq 0 $, then, the equation $ t\geq a \log et + b $ is verified for any $ t\geq 2a \log ea + 2b $ .
\end{lemma}

\begin{proof}
We let $ f(t) = t - a \log et - b $, we have that $ f'(t) \geq 0 $ on $ [a, + \infty) $ and $ f(a) \leq 0 $. Hence $
f(t)=0 $ has a unique solution $ \alpha $ on $ [a, \infty) $ such that $ f(t) \geq 0$ if $t \geq \alpha $. We now focus on
upper bounding $ \alpha $.
The equation $ f(\alpha)= 0  $ is equivalent to
\begin{equation*}
	\log \alpha = \frac{\alpha-b}{a} - 1.
\end{equation*}
Now taking the exponential and reordering this is also equivalent to
\begin{equation*}
	\frac{-\alpha}{a} \exp\left(\frac{-\alpha}{a}\right) = \frac{-\exp \left( -\frac{a+b}{a} \right)}{a}.
\end{equation*}
Let
\begin{align*}
	g: (- \infty, -1]& \longrightarrow [-\frac{1}{e}, 0) \\
x &\longmapsto x e^x .
\end{align*}
The previous equation can be rewritten $ g\left(\frac{-\alpha}{a}\right) = - \frac{\exp \left( -\frac{a+b}{a} \right)}{a}  $. We define $ W_{-1}:[-\frac{1}{e},0) \longrightarrow (- \infty, 1] $ as the(functional) inverse of $g$. The function $g$ is the $ -1 $
branch of the Lambert W function.

We have that for any $ x \leq -1 $,
$ W_{-1}(x e^x) = x $ and that for any $ y\geq e $,  $ -W_{-1}(-\frac{1}{y}) \leq 2 \log(y) $.
	Since g is decreasing on its domain, $ W_{-1} $ is well-defined and decreasing. Moreover, for any $ x\leq -1 $ ,
	$ W_{-1}(g(x)) = x $ .
In particular, we have that $ \alpha = -a W_{-1}\left( - \frac{\exp \left( -\frac{a+b}{a} \right)}{a} \right) $. We will use
that formulation to find an upper bound on $ \alpha $.

We fix some $ y\geq e $. We have $ - 2 \log(y) \leq -1 $ hence $ W_{-1}\left(- 2 \log(y) e^{(- 2 \log(y))}\right) = - 2
\log(y) $, which means that $ 2 \log(y) = - W_{-1}(- \frac{1}{y^{*}}) $ where $ y^{*} = \frac{e^{(2 \log(y))}}{2
\log(y)} = \frac{y^2}{2 \log(y)} $.

Because of the elementary inequality $ 2 \log(x) \leq x $ for $ x>0 $, we conclude that $ y \leq y^{*} $.
Since $ y \mapsto -W_{-1}(-\frac{1}{y}) $ is an increasing function we finally have that for any $ y \geq e $
\begin{equation*}
	-W_{-1}\left(-\frac{1}{y}\right) \leq -W_{-1}\left(- \frac{1}{y^{*}}\right) = 2 \log(y).
\end{equation*}
Applying this to $ y = a \exp \left( \frac{a+b}{a} \right)\geq e $, we get
\begin{equation*}
	\alpha = -aW_{-1}\left(\frac{-1}{y}\right) \leq 2a \log(y) = 2a \log ea + 2b.
\end{equation*}
Since any $ t\geq  \alpha $ will satisfy $ f(t) \geq 0 $, this concludes our proof.

\end{proof}

\begin{lemma}
\label{lemma:sub_gaussian_martingale}
Let $ \theta \in \Theta $, then $ M_t = \exp(\Llik_t(\theta_0) - \Llik_t(\theta)) = \prod_{s=1}^t
\frac{p(Y_t|\theta,A_t)}{p(Y_t|\theta_0,A_t)} $ is a supermartingale with respect to the filtration $ (\F_t)_t $.
\end{lemma}
\begin{proof}
We have 
\begin{align*}
	\EEc{\frac{p(Y_t|\theta, A_t)}{p(Y_t|\theta_0, A_t)}}{\F_{t-1}, A_t} &= \EEc{\exp \left( \frac{(\siprod{\theta_0}{A_t}-
	Y_t)^2 - (\siprod{\theta}{A_t}- Y_t ^ 2)}{2} \right)}{\F_{t-1}, A_t}\\
&= \EEc{\exp \left( \frac{\epsilon_t^2 - (\siprod{\theta - \theta_0}{A_t} - \epsilon_t)^2}{2} \right)}{\F_{t-1}, A_t}\\
&= \exp \left( - \frac{(\siprod{\theta - \theta_0}{A_t})^2}{2} \right)\EEc{\exp \left( \epsilon_t \siprod{\theta -
\theta_0}{A_t} \right)}{\F_{t-1}, A_t}\\
& \leq \exp\left( - \frac{(\siprod{\theta- \theta_0}{A_t})^2}{2} \right) \cdot \exp\left( \frac{(\siprod{\theta-
\theta_0}{A_t})^2}{2} \right)\\
&=1,
\end{align*}
where the inequality comes from the conditional subgaussianity of $ \epsilon_t $.
Finally, by the tower rule of conditional expectations
\begin{equation*}
	\EEc{\frac{p(Y_t|\theta,A_t)}{p(Y_t|\theta_0, A_t)}}{\F_{t-1}} =
\EEc{\EEc{\frac{p(Y_t|\theta,A_t)}{p(Y_t|\theta_0, A_t)}}{\F_{t-1}, A_t}}{\F_{t-1}}  \leq 1.
\end{equation*}

\end{proof}

\subsection{Proof of Proposition~\ref{prop:SIG_TIG}}
\label{app:SIG_TIG}
This is coming from the fact that the mean is the constant minimizing the mean squared error.
We remind the reader of the definition of the surrogate information gain and the true information gain for a policy $
\pi \in \Delta(\A) $ 
\begin{equation}
	\SIG_t(\pi ) = \frac{1}{2} \sum_{a \in \A} \pi(a) \int_{\Theta}(\siprod{\theta - \bar{\theta}(Q_t^{+})}{a} )^2 \, \mathrm{d}Q_t^+(\theta),
\end{equation}
where $ \bar{\theta}(Q_t^{+}) = \EEs{\theta}{\theta \sim Q_t^{+}}$ is the mean parameter under the optimistic posterior
$ Q_t^{+} $.

\begin{equation}
	\IG_t(\pi) = \frac{1}{2} \sum_{a \in \A} \pi(a) \int_{\Theta}(\siprod{\theta}{a} - \siprod{\theta_0}{a})^2 \,
	\mathrm{d}Q_t^{+}(\theta),
\end{equation}
Let's fix $ a \in \A $, we have that
\begin{align*}
	(\siprod{\theta- \theta_0}{a})^2 &= (\siprod{\theta - \bar{\theta}(Q_t^+) + \bar{\theta}(Q_t^+)- \theta_0}{a})^2 \\
					 &= (\siprod{\theta - \bar{\theta}(Q_t^+)}{a})^2 + 2\siprod{\theta-
					 \bar{\theta}(Q_t^+)}{a} \siprod{\bar{\theta}(Q_t^+) - \theta_0}{a} +
					 (\siprod{\bar{\theta}(Q_t^+)- \theta_0}{a})^2\\
					 &\geq (\siprod{\theta - \bar{\theta}(Q_t^+)}{a})^2 + 2\siprod{\theta-
					 \bar{\theta}(Q_t^{+})}{a} \siprod{\bar{\theta}(Q_t^{+}) - \theta_0}{a}
\end{align*}
Now using that $ \bar{\theta}(Q_t^{+}) = \int_{\Theta} \theta \, dQ_t^{+}(\theta) $ and integrating, we get

\begin{equation*}
	\int_\Theta (\siprod{\theta - \theta_0}{a})^2 \, \mathrm{d}Q_t^{+}(\theta) \geq \int_\Theta (\siprod{\theta -
	\bar{\theta}(Q_t^{+})}{a})^2 \, \mathrm{d}Q_t^{+}(\theta).	
\end{equation*}

Multiplying by $ \pi(a) $ and summing over actions, we get the claim of the lemma.

\subsection{Generalization of the AM-GM inequality}
\label{app:gen_am_gm}

Dealing with the generalized information ratio requires bounding the cubic root of products. While one could use
Hölder's inequality to deal directly with products, we find it more flexible to use a variational form of this inequality.
In all that follows, we let $ p>1 $ be a real number and $ q $ be such that $ \frac{1}{p} + \frac{1}{q} = 1 $. It is not
hard to check that $ q = \frac{p}{p-1} $.
We start by stating a direct consequence of the Fenchel-Young Inequality which can be seen as an extension of the AM-GM
inequality.
\begin{lemma}
\label{lemma:FY_xp}
	Let $ x,y \geq 0 $, then 
\begin{equation}
\label{eq:FY_xp}
	xy \leq \frac{x^p}{p} + \frac{y^q}{q}.
\end{equation}
With equality if and only if $ p x^{p-1} = y $
\begin{proof}
	One can check that the Fenchel dual of the function 
\begin{align*}
f:& \real^{+} \longrightarrow \real^{+} \\
&x \longmapsto \frac{x^p}{p} 
\end{align*}
is exactly $ f^{*}(y) = \frac{1}{q} y^q $ (for non-negative $y$). Then the Lemma is a direct consequence of the Fenchel Young
inequality and of its equality case.
\end{proof}

\end{lemma}

Refining a bit this Lemma, we get the following variational form of the previous inequality :
\begin{lemma}
\label{lemma:Gen_AM_GM}
Let $ x,y \geq 0, \lambda > 0 $, then
\begin{equation}
\label{eq:Gen_AM_GM}
	\sqrt[p]{xy} \leq \frac{x}{\lambda} + c_p^{*} (\lambda y)^{\frac{1}{p-1}}
\end{equation}
where $ c_p^{*} = (p-1)\frac{1}{p} ^{\frac{p}{p-1}} $ with equality if and only if $ x=y=0 $ or $ \lambda = p
\frac{x^{\frac{p-1}{p}}}{y^{\frac{1}{p}}} $.
\end{lemma}
\begin{proof}
	We apply the previous lemma to $ \sqrt[p]{\frac{p x}{\lambda}} $ and $ \sqrt[p]{\frac{\lambda y}{p}} $.
\end{proof}

In order to go from the variational form to the product form, we may use the following result.

\begin{lemma}
\label{lemma:reverse_AM_GM}
Let $ \alpha, \beta >0  $, then 
\begin{equation}
\label{eq:reverse_AM_GM}
	\inf_{\lambda >0} \frac{\alpha}{\lambda} + \beta \lambda^{\frac{1}{p-1}} = c_p \alpha^{\frac{1}{p}}
	\beta^{\frac{p-1}{p}},
\end{equation}
where $ c_p = p \frac{1}{p-1}^{\frac{p-1}{p}} $ satisfies $ c_p \cdot {c_p^*}^{\frac{p-1}{p}} =1 $, and the minimum is
reached at $ \lambda^{*} = (p-1)^{\frac{p-1}{p}} \frac{\alpha^{\frac{p-1}{p}}}{\beta^{\frac{p-1}{p}}} $.
	
\end{lemma}
\begin{proof}
	Applying the previous Lemma to $ x= \alpha $ and $ y = c_p^{\frac{p}{p-1}} \beta^{p-1} $ yields the result.
\end{proof}

\paragraph{Remark}
An alternative is to pick $ \lambda $ to make both terms equal resulting in the same result but with 2 as a leading
constant. Now 
\begin{align*}
	c_p &= p^{\frac{1}{p}} \frac{p}{p-1}^{\frac{p-1}{p}}\\
	    &= \exp{\left( \frac{1}{p} \log p + \frac{p-1}{p} \log \frac{p}{p-1} \right)}\\
	    &\leq \frac{1}{p}\cdot p + \frac{p-1}{p} \cdot \frac{p}{p-1}\\
	    &= 2.
\end{align*}
With equality if and only if $ p=2 $. So, the choice of $ c_p $ always yields a better leading constant. However, $ c_3 \simeq 1.88 $ so one could argue
that the gain is small. Since we will usually use Lemma~\ref{lemma:Gen_AM_GM}, $ c_p^{*} $ will naturally
appear and $ c_p $ will cancel it, ultimately making the leading constant as simple as possible.

\section{Experimental details}
\label{sec:experiment_details}

Here, we describe our implementation of the SOIDS algorithm in more detail, as well as the hyperparameters of all the methods used in our experiments. To run the SOIDS algorithm, one must minimise $\SIRs_t(\pi)$ w.r.t.\ $\pi$ in each round $t$. This is not straightforward, because $\SIRs_t(\pi)$ contains expectations w.r.t.\ the optimistic posterior $Q_t^+$. When we use the Spike-and-Slab prior in Appendix \ref{app:prior_comparator_bound}, we are not aware of any efficient method that can be used to maximise $\SIRs_t(\pi)$. Instead, we draw (approximate) samples $\theta^{(1)}, \dots, \theta^{(M)}$ from $Q_t^+$ to produce the estimates $\wt\Delta_t(\pi)$ and $\wt\IG_t(\pi)$ for the surrogate regret and the surrogate information respectively, where
\begin{equation*}
\wt\Delta_t(\pi) = \sum_{a \in \A}\pi(a)\frac{1}{M}\sum_{i=1}^{M}\Delta(a, \theta^{(i)}), \qquad \wt\IG_t(\pi) = \frac{1}{2}\sum_{a \in \A}\pi(a)\frac{1}{M}\sum_{i=1}^{M}\big(\siprod{\theta^{(i)} - \bar{\theta}_M}{a}\big)^2\,.
\end{equation*}

Here, $\bar{\theta}_M$ is the sample mean $\frac{1}{M}\sum_{i=1}^{M}\theta^{(i)}$. We then maximimse the approximate surrogate information ratio $\wt\IR_t^{(2)}(\pi)$, where
\begin{equation*}
\wt\IR_t^{(2)}(\pi) = \frac{(\wt\Delta_t(\pi))^2}{\wt\IG_t(\pi)}\,.
\end{equation*}

To draw the samples $\theta^{(1)}, \dots, \theta^{(M)}$, we use the empirical Bayesian sparse sampling procedure proposed by \citet{Hao_L_D21a}, which is designed to draw samples from the Bayesian posterior. To sample from the optimistic posterior, we incorporate the optimistic adjustment into the likelihood. This method replaces the theoretically sound spike-and-slab prior with a relaxation in which the ``spikes'' are Laplace distributions with small variance, and the ``slabs'' are Gaussian distributions with large variance. In particular, the density of this prior is
\begin{equation*}
q_1(\theta) = \sum_{\gamma \in \{0, 1\}^d}p(\gamma)\prod_{j=1}^{d}[\gamma_j\psi_1(\theta_j) + (1 - \gamma_j)\psi_0(\theta_j)]\,.
\end{equation*}

Here, $\psi_1(\theta)$ is the density function of a univariate Gaussian distribution, with mean 0 and variance $\rho_1$, and $\psi_0$ is the density function of a univariate Laplace distribution, with mean 0 and scale parameter $\rho_0$. $p(\gamma)$ is a product of Bernoulli distributions with mean $\beta$. In our experiments, we always use $\rho_1 = 10$, $\rho_0 = 0.1$ and $\beta = 0.1$. Also, we set the learning rates to $\eta = 1/2$ and $\lambda_t = \min(\frac{1}{2}, \frac{1}{10}\max(\sqrt{\frac{s\log(edt/s)}{dt}}, (\frac{\log(edt/s)}{t})^{2/3}))$.

Implementing the OTCS baseline exactly would require us to compute the means of the distributions played by an exponentially weighted average forecaster with a sparsity prior. These distributions are the same as the optimistic posterior, except $\lambda_t = 0$ (i.e.\ there is no optimistic adjustment). In our implementation of the OTCS baseline, we draw samples using the same empirical Bayesian sparse sampling procedure, and then replace the exact means with the sample means. We use the same choices for the parameters $\eta$, $\rho_1$, $\rho_0$ and $\beta$. We set the radii of the confidence sets to the values given in Theorem 4.7 of \citet{clerico2025confidence}

For the LinUCB baseline, we set the radii of the confidence sets to the values given in Theorem 2 of \citet{abbasi2011improved}. For the ESTC baseline, we set the exploration length $T_1$ to $50$ when $d=20$, $100$ when $d = 40$ and $d=100$. These values were chosen based on a small amount of trial and error. The theoretically motivated values in Theorem 4.2 of \citet{Hao_L_W20} are much larger than these values. Also for ESTC, we set the LASSO regularisation parameter to $\lambda = 4\sqrt{\log(d)/T_1}$, which is the value given in Theorem 4.2 of \citet{Hao_L_W20}.
\end{document}